\documentclass[10pt]{article} 
\usepackage[preprint]{tmlr}

\usepackage{mathtools,amssymb,amsthm}
\usepackage{thm-restate}
\usepackage{slashed}

\usepackage{xcolor}
\usepackage{tikz}
\usetikzlibrary{babel, arrows, decorations, decorations.pathmorphing,
  decorations.pathreplacing, decorations.shapes, calc, arrows.meta,
  positioning, decorations.markings, fit, shapes.geometric, backgrounds}

\usepackage{microtype}
\usepackage{url}
\usepackage{hyperref}

\newtheorem{theorem}{Theorem}

\newtheorem{corollary}{Corollary}
\newtheorem{proposition}{Proposition}

\newtheorem{definition}{Definition}
\newtheorem{example}{Example}

\newtheorem{assumption}{Assumption}

\usepackage{slashed}
\newcommand{\ind}{{\perp\!\!\!\perp}}

\newcommand{\notind}{\slashed{\ind}}
\newcommand{\V}{\mathcal{V}}
\newcommand{\F}{\mathcal{F}}
\newcommand{\X}{\mathbf{X}}
\newcommand{\E}{\mathcal{E}}
\newcommand{\Y}{\mathbf{Y}}
\newcommand{\Z}{\mathbf{Z}}
\newcommand{\W}{\mathbf{W}}
\newcommand{\G}{{\mathcal{G}}}
\newcommand{\M}{{\mathcal{M}}}
\newcommand{\C}{\mathcal{C}}
\newcommand{\U}{\mathcal{U}}
\newcommand{\undirectedarrow}{\!\mathrel{\raisebox{0.51ex}{\rule{.9em}{0.14ex}}}\!}

\DeclareMathOperator{\Do}{do}
\DeclareMathOperator{\Forb}{Forb}
\DeclareMathOperator{\Ch}{Ch}
\DeclareMathOperator{\Pa}{Pa}
\DeclareMathOperator{\Anc}{Anc}
\DeclareMathOperator{\Desc}{Desc}

\definecolor{coolorange}{HTML}{F28E00}

\newcommand{\dashleftrightarrow}{%
  \mathrel{\tikz[baseline=-0.5874ex]{
        \draw[
          <->,
          dash pattern=on 1.25pt off 1.25pt,
          dash phase=1.4pt
        ] (0,0) -- (.925em,0);
  }}%
}

\title{Causal Effect Identifiability in Causal Abstractions: A Strict Hierarchy of Criteria}

\author{\name Clément Yvernes \email clement.yvernes@univ-grenoble-alpes.fr \\
      \addr Univ Grenoble Alpes, CNRS, Grenoble INP, LIG\\
      \AND
      \name Marianne Clausel \email marianne.clausel@univ-lorraine.fr \\
      \addr Université de Lorraine, CNRS, CRAN, 54000 Nancy\\
      \AND
      \name Emilie Devijver \email emilie.devijver@univ-grenoble-alpes.fr\\
      \addr Univ Grenoble Alpes, CNRS, Grenoble INP, LIG\\
      \AND Eric Gaussier \email eric.gaussier@imag.fr \\
      \addr Univ Grenoble Alpes, CNRS, Grenoble INP, LIG
      }

\def\month{MM}  
\def\year{YYYY} 
\def\openreview{\url{https://openreview.net/forum?id=XXXX}} 

\begin{document}

\maketitle

\begin{abstract}
    Identifying causal effects from observational data typically requires a fully specified causal graph, an assumption that is often unrealistic in complex or high-dimensional settings. Causal abstractions address this limitation by retaining only partial causal information, thereby defining equivalence classes of compatible causal graphs. In this paper, we study causal effect identifiability directly at the level of these graph collections, rather than at the level of individual causal diagrams. We formalize several notions of identifiability in this setting and organize them into a hierarchy that clarifies their relationships and differences. In particular, we establish that a causal effect may be identified over a collection of graphs without admitting a single causal derivation that holds uniformly across all compatible graphs. These results help reveal gaps and clarify ambiguities in the existing literature.
\end{abstract}

\section{Introduction}\label{sec:intro}

Knowing the effect of a treatment $\mathbf{X}$ on an outcome $\mathbf{Y}$, represented by the interventional distribution $P(\mathbf{Y} \mid do(\mathbf{X}))$, is central to many applications. Yet, performing interventions is often impractical due to ethical concerns, potential harm, or prohibitive costs. In such cases, causal effect identification aims to determine whether interventional quantities can be expressed solely in terms of observational data \citep{pearl_causality_2009}. This identifiability problem is typically addressed by deriving graphical criteria on causal diagrams and corresponding do-free formulas. However, these approaches generally assume that the underlying causal diagram is fully specified, requiring prior knowledge of the relationships between all observed variables. In practice, such knowledge is rarely available, especially in complex or high-dimensional settings, which substantially limits the applicability of standard causal inference methods.

A more relaxed assumption than having access to a fully specified causal graph is to rely on a causal abstraction, a representation that captures partial causal information about the underlying system. Rather than fixing a single graph, such abstractions induce classes of compatible causal diagrams and allow reasoning at a higher level of description. However, each compatible causal diagram may induce different answers to the same interventional query. As a result, identifiability is no longer a unique notion, but depends on how one aggregates or quantifies validity across the collection of graphs. While several such notions have been proposed, they are usually developed independently and lack a common framework for comparison.

Causal abstractions have been studied from multiple complementary perspectives. On the one hand, several works focus on their theoretical properties, including compositionality \citep{pmlr-v161-rischel21a}, the computation of optimal abstractions \citep{zennaro2022computingoptimalabstractionstructural}, the use of macro-variables or grouped variables as substitutes for micro-level variables \citep{PARVIAINEN2017110, DBLP:journals/corr/abs-1812-03789}, and the quantification of information loss relative to the underlying causal model \citep{10.24963/ijcai.2023/638}. On the other hand, causal abstractions have also been used as a tool in a variety of downstream settings, including causal discovery \citep{WahlNinadRunge+2024}, query identifiability on clusters of variable \citep{anand_causal_2023, yvernes2025relaxing}, reinforcement learning \citep{10.5555/3702676.3702869}, time series analysis \citep{assaad2022, 10.5555/3702676.3702683, multivariateICA}, neural network analysis \citep{NEURIPS2021_4f5c422f}, mechanistic interpretability \citep{geiger2025causalabstractiontheoreticalfoundation, ICLR2025_51157713}, and hierarchical modeling \citep{weinstein2024hierarchicalcausalmodels}. Despite this broad range of studies, existing notions of identifiability under causal abstractions remain largely developed in isolation, making them difficult to compare within a unified framework.

We address this gap by unifying several notions of identifiability over collections of causal diagrams. These notions form a strict hierarchy, allowing us to compare their strength and determine when one strictly subsumes another. This hierarchy can assess, for a given abstraction, whether existing theory misses identifiable queries or fully captures the identifiable ones. More precisely, our contributions are as follows:
\begin{itemize}
    \item We unify several notions of identifiability over collections of causal graphs and establish a strict hierarchy between them, enabling systematic comparison of their relative strength,
    \item In particular, we show that
    \begin{itemize}
        \item We establish that adjustment-based identifiability (ICA) does not imply uniform do-calculus identifiability (ICD), complementing the known result that ICD does not imply ICA.
        \item We prove that graphical identifiability (IG) strictly subsumes ICD, resolving the conjecture of \citet{yvernes2025identifiability}.
    \end{itemize}
\end{itemize}

The remainder of the paper is organized as follows. Section~\ref{sec:preliminaries} provides the necessary background, Section~\ref{sec:collectionGraphs} introduces identifiability notions for collections of causal diagrams, and Section~\ref{sec:strict_hierarchy} establishes their strict hierarchy. Section~\ref{sec:discussions} interprets this hierarchy and shows that, under a mild assumption, identifiability over a collection is determined by its maximal elements, allowing the collection to be reduced without loss of identifiability. Section~\ref{sec:conclusion} concludes, and the Appendix contains all the proofs.

\section{Preliminaries}\label{sec:preliminaries}
\subsection{Foundations of Causal Modeling}
\paragraph{Graphs.} 
A \emph{mixed graph} $\G = (\mathcal{V}, \mathcal{E})$ is a graph that contains two type of arrows, \emph{directed arrows} ($\rightarrow$) and \emph{dashed-bidirected arrows} ($\dashleftrightarrow$).
For a mixed graph $\G$, if $X\rightarrow Y$, then  $X$ is a \emph{parent} of $Y$ and $Y$ is a \emph{child} of $X$. A \emph{path} is a sequence of distinct vertices such that there exists an arrow in $\G$ connecting each element to its successor. A \emph{directed path} (or \emph{causal path}) is a path consisting only of directed edges ($\rightarrow$), such that all arrows are oriented towards the last vertex.
If there is a directed path from $X$ to $Y$, then $X$ is an \emph{ancestor} of $Y$, and $Y$ is a \emph{descendant} of $X$.\footnote{Note that a vertex is a descendant and an ancestor of itself, since for all variable  $V$ the path containing $V$ alone is a directed path from $V$ to $V$.} The sets of parents, children, ancestors and descendants of $X$ in $\G$ are denoted by $\Pa(X,\G)$, $\Ch(X,\G)$, $\Anc(X,\G)$ and $\Desc(X,\G)$ respectively. 
For two disjoint subsets \(\X, \Y \subseteq \V\), a \textit{path} from $\X$ to $\Y$ is a path from some \(X \in \X\) to some  \(Y \in \Y\). A path from $\X$ to $\Y$ is \textit{proper} if only its first node is in $\X$.
A \emph{directed cycle} is a circular list of distinct vertices in which each vertex is a parent of its successor. When a (directed) mixed graph does not contain any directed cycle, then this graph is called an \emph{Acyclic Directed Mixed Graph (ADMG).} When an ADMG does not contain any dashed-bidirected arrow, it is called a \emph{Directed Acyclic Graph (DAG)}. A graph \( \G_1 = (\V_1, \E_1) \) is said to be a subgraph of another graph \( \G_2 = (\V_2, \E_2) \) if \( \V_1 \subseteq \V_2 \) and \( \E_1 \subseteq \E_2 \). In this case, we write \( \G_1 \subseteq \G_2 \), which defines a partial order over graphs.

On a path $\pi$, a vertex $V$ is said to be a \emph{collider} on $\pi$ if it has two incoming arrows in $\pi$ (i.e., $\rightarrow V \leftarrow$, or $\rightarrow V \dashleftrightarrow$, or $\dashleftrightarrow V \leftarrow$ or $\dashleftrightarrow V \dashleftrightarrow$). Otherwise, it is said to be a \emph{non-collider} on $\pi$.
A vertex $V$ is said to be \emph{active} on a path relative to $\Z$ if 1) $V$ is a collider and one of its descendants is in $\Z$ or 2) $V$ is a non-collider and is not in $\Z$. A path $\pi$ is said to be \emph{active} given (or conditioned on) $\Z$ if every vertex on $\pi$ is active relative to $\Z$. Otherwise, $\pi$ is said to be \emph{inactive} given $\Z$. Given a graph $\mathcal{G}$, the sets $\X$ and $\Y$ are said to be \emph{d-separated} by $\Z$ if every path between $\X$ and $\Y$ is inactive given $\Z$. We denote this by $\X \ind_{\G} \Y \mid \Z$. Otherwise, $\X$ and $\Y$ are said to be \emph{d-connected} given $\Z$, which we denote by $\X \notind_{\G} \Y \mid \Z$.

The mutilated graph ${\mathcal{G}}_{\overline{\X}\underline{\Z}}$ is obtained from $\G$ by removing all edges with an arrowhead into $\X$ (i.e., $A \rightarrow \X$ or $A \dashleftrightarrow \X$, for any vertex $A$) 
and all edges with a tail out of $\Z$ (i.e., $\Z \rightarrow A$, for any vertex $A$).

\paragraph{Structural Causal Models.} A \emph{{Structural Causal Model} (SCM)} $\mathcal{M}$ consists of a 4-tuple $\langle \U, \V, \F_\M, P(\U)\rangle$, where $\U$ is a set of exogenous (latent) mutually independent variables and $\V$ is a set of endogenous (measured) variables. $\F_\M$ is a collection of functions $\{f_i\}_{i=1}^{|\V|}$ such that each endogenous variable $V_i\in\V$ is a function $f_i\in\F_\M$ of $\U_i\cup \Pa(V_i)$, 
where $\U_i\subseteq\U$ and $\Pa(V_i)\subseteq\V\setminus \{V_i\}$. The uncertainty is encoded through a probability distribution over the exogenous variables, $P(\U)$. The resulting joint distribution over \(\V\), obtained by marginalizing out \(\U\), is called the \emph{observational distribution} and is denoted \(P_{\mathcal{M}}\left(\V\right)\).
Each SCM $\mathcal{M}$ induces a mixed graph -- $\G_{\mathcal{M}} \coloneqq (\V, \E = (\E_{D},\E_{B}))$, known as a \emph{causal diagram}, that encodes the structural relations among $\V\cup\U$, where every endogenous variable $V_i \in \V$ is a vertex. 
$\G_{\mathcal{M}}$ contains directed edges which connect each variable $V_i \in \mathcal{V}$ to its parents $V_j \in \Pa(V_i)$ as $(V_j \rightarrow V_i)$, and bidirected edges between variables $V_i, V_j \in \mathcal{V}$ that share a common exogenous parent, i.e., such that $\mathcal{U}_i \cap \mathcal{U}_j \neq \emptyset$, denoted $V_i \dashleftrightarrow V_j$. Throughout this paper, we restrict attention to acyclic causal models, meaning that the associated causal diagrams are ADMGs.

\paragraph{Interventions.}
In a given SCM \(\mathcal{M}=\langle \U,\V,\mathcal{F},P(\U)\rangle\), let \(\mathbf{V}=(V_1,\dots,V_n) \) be a subset of \(\V\) and \(\mathbf v\) be a valuation of \(\mathcal V\). An intervention \(\Do(\mathbf{V}=\mathbf{v})\) replaces each structural function \(f_i\) for \(V_i\in\mathbf{V}\) by the constant assignment \(V_i:=v_i\), which in the causal diagram \(\G_{\mathcal{M}}\) corresponds to removing all incoming edges into each \(V_i\). The resulting joint distribution over \(\V\), obtained by marginalizing out \(\U\), is called the \emph{interventional distribution} and is denoted \(P_{\mathcal{M}}\left(\V\mid\Do(V =\mathbf{v})\right)\) or \(P_{\mathcal{M}}\left(\V\mid\Do(\mathbf{v})\right)\) for short. For any pairwise disjoint subsets $\Y,\X,\W \subseteq \V$, with $\Y \neq \emptyset$ and corresponding valuations $\mathbf y, \mathbf x, \mathbf w$, we will denote $P(\Y = \mathbf y \mid \Do(\X = \mathbf x), \mathbf X = \mathbf w)$ by $P(\mathbf y \mid \Do( \mathbf x), \mathbf w)$.

\paragraph{Compatibility}
A probability distribution $P$ is said to be \emph{compatible} with an ADMG $\mathcal{G}$ if there exists a SCM inducing $\mathcal{G}$ whose observational distribution is $P$. In particular, any distribution compatible with $\mathcal{G}$ is Markovian with respect to $\mathcal{G}$ (Proposition 6.31 in \citet{PetJanSch17}), that is, it satisfies the global Markov property, i.e., every d-separation relation in $\mathcal{G}$ implies the corresponding conditional independence in $P$.

\paragraph{Causal Queries \& Identification Functionals.}
Let $\V$ denote the set of observable variables and let $\X,\Y,\Z\subseteq\V$ be pairwise disjoint, with $\Y \neq \emptyset$. A \emph{causal query} is an expression of the form \(P(\mathbf{y}\mid\Do(\mathbf{x}),\mathbf{z})\). An \emph{identification functional} is a map \(\phi\) from joint distributions over $\V$ to $\mathbb{R}^+$. Given an SCM $\mathcal{M}$, $\phi$ is \emph{valid} for \(P(\mathbf{y}\mid\Do(\mathbf{x}),\mathbf{z})\) if
\(
\phi\left(P_{\mathcal{M}}(\V)\right)
=
P_{\mathcal{M}}\!\left(\mathbf{y}\mid\Do(\mathbf{x}),\mathbf{z}\right).
\)

All probability expressions and equalities in this paper are understood on their natural domains of definition.
\subsection{Identifiability in a Single Graph}
\label{sec:singleGraph} 
Different SCMs may induce the same causal diagram which represents an equivalence class of SCMs. 
A causal query \(P\bigl(\mathbf{y} \mid \Do(\mathbf{x}), \mathbf{z}\bigr)\) is said to be \emph{identifiable} in a graph $\G$ if all SCMs that induce $\G$ and the same observational distribution agree on the value of the query. Definition~\ref{def:single} formalizes this notion.

\begin{definition}[Identifiability in a Single Graph]\label{def:single}
    Let $\G=(\V,\E)$ be a causal diagram and let $\X,\Y,\Z\subseteq\V$ be pairwise disjoint, with \(\Y \neq \emptyset\). The causal query \(P\bigl(\mathbf{y}\mid\Do(\mathbf{x}),\mathbf{z}\bigr)\) is \emph{identifiable} in $\G$ if there exists an identification functional $\phi$ valid for \(P\bigl(\mathbf{y}\mid\Do(\mathbf{x}),\mathbf{z}\bigr)\) in every SCM $\M$ satisfying $\G_{\M}=\G$. In this case, $\phi$ is said to be valid for \(P\bigl(\mathbf{y}\mid\Do(\mathbf{x}),\mathbf{z}\bigr)\) in $\G$.
\end{definition}

Definition~\ref{def:single} formalizes identifiability as the existence of a valid identification functional for a given 
causal query. It is equivalent to Definition~3.2.3 in~\citet{pearl_causality_2009}, rewritten from a functional perspective. The following example illustrates this notion on a simple causal graph that we shall use later as a running example.

\begin{example}
\label{ex:identification}
Consider the DAG $\G \coloneqq$
\begin{tikzpicture}[ ->, >=stealth, circle, baseline={(6.2,-.1)},transform shape]
    \node (X) at (0, 0) {X};
    \node (Y) at (1, 0) {Y};
    \node (Z) at (-1, 0) {Z};
    \draw (X) -- (Y);
    \draw (Z) -- (X);
    \path (Z) edge[bend left=25] (Y);
\end{tikzpicture} 
and the query $P(y \mid \Do(x))$. Then the identification functional
\(
    \phi : P(\mathcal V) \mapsto  \sum_{z}P(y \mid x, z) P(z) 
\) is valid for $P(y \mid \Do(x))$ in $\G$. The derivation of this functional will be given later.
\end{example}

However, this definition is purely semantic: it does not directly provide a procedure for determining whether a query is identifiable, nor for constructing an identification functional when one exists. Addressing these questions requires systematic tools for transforming causal queries into valid identification functional. Pearl’s do-calculus provides such a framework through three transformation rules that enable reasoning about interventions directly from the graphical structure of the causal diagram. We restate these rules in Theorem~\ref{th:do_calculus}.

\begin{theorem}[Do-calculus rules, \cite{pearl_causality_2009}]
\label{th:do_calculus}
Let $\mathcal{G} = (\mathcal{V}, \mathcal{E})$ be an ADMG, and let $\mathbf{X,Y, Z, W}$ be pairwise disjoint subsets of $\mathcal{V}$, with $\Y \neq \emptyset$. Then the following rules hold for all distributions compatible with $\mathcal{G}$:
\begin{itemize}
    \item[$\mathcal{R}1$ \!] \text{insertion/deletion of observations:}\newline
        $\Y \ind_{\G_{\overline{\X}}} \Z \mid \X,\W 
        \Rightarrow P(\mathbf{y} \mid \Do(\mathbf{x}), \mathbf{z}, \mathbf{w})
        = P(\mathbf{y} \mid \Do(\mathbf{x}), \mathbf{w})$,
    \item[$\mathcal{R}2$ \!] \text{action/observation exchange:}\newline
        $\mathbf{Y} \ind_{\mathcal{G}_{\overline{\mathbf{X}}\, \underline{\mathbf{Z}}}} \mathbf{Z} \mid \mathbf{X}, \mathbf{W} \Rightarrow P(\mathbf{y} \mid \Do(\mathbf{x},\mathbf{z}), \mathbf{w})
        = P(\mathbf{y} \mid \Do(\mathbf{x}), \mathbf{z}, \mathbf{w})$,
    \item[$\mathcal{R}3$ \!] \text{insertion/deletion of actions :}\newline
        $\mathbf{Y} \ind_{\mathcal{G}_{\overline{\mathbf{X}}\, \overline{\mathbf{Z(W)}}}} \mathbf{Z} \mid \mathbf{X}, \mathbf{W} \Rightarrow P(\mathbf{y} \mid \Do(\mathbf{x},\mathbf{z}), \mathbf{w})
        = P(\mathbf{y} \mid \Do(\mathbf{x}), \mathbf{w})$,
\end{itemize}
where $\mathbf{Z}(\mathbf{W})$ is the set of $\mathbf{Z}$-nodes non-ancestors of $\mathbf{W}$ in $\G_{\overline{\mathbf X}}$.
\end{theorem}

The do-calculus proceeds by repeatedly applying its rules to transform observational and interventional probabilities of the form \(P(\mathbf{y}\mid\mathrm{do}(\mathbf{x}),\mathbf{w})\). Starting from a causal quantity of interest, each application replaces the current expression with an equivalent one, with the validity of the transformation guaranteed by a graphical condition in the causal graph. Combining applications of do-calculus rules with standard probability manipulations, such as Bayes' rule and marginalization, enables causal reasoning.

In this paper, a \emph{causal derivation} is any finite sequence of applications of do-calculus rules, possibly combined with standard probability manipulations, that transforms a causal quantity into another one. Such a derivation is said to \emph{hold in} a causal diagram~\(\G\) if every d-separation condition required by the applied do-calculus rules is satisfied in~\(\G\). Standard probability manipulations are always interpreted within their usual domain of validity, and do not affect the notion of validity of the derivation. Given a set of variables \(\mathcal V\), a derivation is said to be \emph{on} \(\mathcal V\) if every variable appearing in any expression throughout the derivation belongs to \(\mathcal V\).

When a causal derivation rewrites a causal query \( P\bigl(\mathbf{y} \mid \Do(\mathbf{x}), \mathbf{z}\bigr) \) into a do-free expression \( \phi\bigl(P(\V)\bigr) \), \( \phi\bigl(P(\V)\bigr) \) depends only on the observational distribution $P(\V)$ and \(\phi\) is a valid identification functional for the query. In this case, $\phi$ is called an \emph{identification formula} and the causal derivation is said to \textit{identify}  \( P\bigl(\mathbf{y} \mid \Do(\mathbf{x}), \mathbf{z}\bigr) \). Example~\ref{ex:identification_avec_derivation} illustrates a causal derivation identifying a causal query in a given graph.

\begin{example}
\label{ex:identification_avec_derivation}
Consider the DAG $\G \coloneqq$\!\!\!
\begin{tikzpicture}[ ->, >=stealth, circle, baseline={(6.2,-.1)},transform shape]
    \node (X) at (0, 0) {X};
    \node (Y) at (1, 0) {Y};
    \node (Z) at (-1, 0) {Z};
    \draw (X) -- (Y);
    \draw (Z) -- (X);
    \path (Z) edge[bend left=25] (Y);
\end{tikzpicture}\!\!. 
Then $P(y \mid \Do(x))$ is identified by the following causal derivation:
\begin{align*}
    P(y \mid \Do(x)) &= \sum_{z}P(y,z \mid \Do(x))  &(\text{by marginalisation})\\
        &= \sum_{z}P(y \mid \Do(x), z) P(z \mid \Do(x)) &(\text{by chain rule})\\
        &= \sum_{z}P(y \mid \Do(x), z) P(z) &(\text{by } \mathcal R3)\\
        &= \sum_{z}P(y \mid x, z) P(z) &(\text{by } \mathcal R2)
\end{align*}
In this case, we recover Example~\ref{ex:identification} and an example of backdoor adjustment \citep{pearl_causality_2009}.
\end{example}

Do-calculus is sound and complete for causal identification in causal diagrams~\citep{shpitser_identification_2006}. Consequently, a causal query is identifiable in~\(\G\) if and only if there exists a causal derivation that rewrites the query into a do-free expression and holds in~\(\G\). Whenever such a derivation exists, it yields an identification formula for the query in~\(\G\). This principle underlies the \textsc{ID} algorithm~\citep{shpitser_identification_2006}, which systematically constructs such a derivation whenever the query is identifiable.

In practice, identifiability can also be established through more direct \emph{graphical criteria}, namely conditions expressed solely in terms of graphical properties that guarantee the existence of a valid identification functional. A standard example is the backdoor criterion \citep{pearl_causality_2009}.

\begin{definition}[Backdoor criterion \citep{pearl_causality_2009}]
\label{def:backdoor}
    Let $X$ and $Y$ be two variables in a causal graph. A set of variables $\mathbf Z$ satisfies the backdoor criterion with respect to $(X,Y)$ if:
    \begin{enumerate}
        \item No element of $\mathbf Z$ is a descendant of $X$, and
        \item $\mathbf Z$ blocks all paths between $X$ and $Y$ that contain an arrow into $X$.
    \end{enumerate}
\end{definition}

\begin{theorem}[Backdoor adjustment \citep{pearl_causality_2009}]
\label{th:backdoor_adjustment}
    Let $X$ and $Y$ be two variables in a causal graph. If a set of variables $\mathbf Z$ satisfies the backdoor criterion with respect to $(X,Y)$, then the causal effect of $X$ on $Y$ is identifiable and is given by
    
    $$
    P(y \mid \Do(x))
    =
    \sum_{\mathbf z} P(y \mid x,\mathbf z)P(\mathbf z).
    $$
\end{theorem}

When no set satisfying the backdoor criterion exists, the causal effect may still be identifiable through an alternative graphical criterion. The frontdoor criterion \citep{pearl_causality_2009} provides a classical example of such an identification strategy. Rather than adjusting directly for variables that block backdoor paths from $X$ to $Y$, it exploits variables lying on the directed paths from $X$ to $Y$ to identify the effect indirectly.

\begin{definition}[Frontdoor criterion \citep{pearl_causality_2009}]
\label{def:frontdoor}
    Let $X$ and $Y$ be two variables in a causal graph. A set of variables $\mathbf Z$ satisfies the frontdoor criterion with respect to $(X,Y)$ if:
    \begin{enumerate}
    \item $\mathbf Z$ intercepts all directed paths from $X$ to $Y$;
    \item there is no unblocked backdoor path from $X$ to $\mathbf Z$; and
    \item $X$ blocks all paths from $\mathbf Z$ to $Y$ that contain an arrow into any variable in $\mathbf Z$.
    \end{enumerate}
\end{definition}

\begin{theorem}[Frontdoor adjustment \citep{pearl_causality_2009}]
\label{th:frontdoor_adjustment}
    Let $X$ and $Y$ be two variables in a causal graph. If a set of variables $\mathbf Z$ satisfies the frontdoor criterion with respect to $(X,Y)$ and $P(x,\mathbf z)>0$, then the causal effect of $X$ on $Y$ is identifiable and is given by
    
    $$
    P(y \mid \Do(x))
    =
    \sum_{\mathbf z} P(\mathbf z \mid x)
    \sum_{x'} P(y \mid \mathbf z,x')P(x').
    $$
\end{theorem}

The backdoor criterion also admits a more general formulation. We therefore recall the generalized adjustment criterion of \citet{Perkovic_2015}, which characterizes valid adjustment sets; see also \citet{pmlr-vR12-zander14a}.

\begin{definition}[Generalized Adjustment Criterion]
\label{def:critere_ajustement}
Let $\mathbf{X}, \mathbf{Y}$ and $\mathbf{Z}$ be pairwise distinct sets of variables in a causal diagram $\G$.
$\mathbf{Z}$ is said to satisfy the \emph{adjustment criterion} relative to $\mathbf{X}$ and $\mathbf{Y}$ in $\G$ if:
\begin{enumerate}
    \item \( \Forb\left(\mathbf{X}, \mathbf{Y}, \G\right) \cap \mathbf{Z} = \emptyset\); and
    \item $\mathbf{Z}$ blocks all proper non-causal paths from $\mathbf{X}$ to $\mathbf{Y}$ in $\G$.
\end{enumerate}
where the \emph{forbidden set} \( \Forb\left(\mathbf{X}, \mathbf{Y}, \G\right)\) is the set of all descendants of any \( W \notin \mathbf{X}\) which lies on a proper causal path from $\mathbf{X}$ to $\mathbf{Y}$.
\end{definition}


\begin{theorem}[Adjustment\citep{Perkovic_2015}]
\label{th:adjustment_perkovic}

    Let $\mathbf{X}, \mathbf{Y}$ and $\mathbf{Z}$ be pairwise distinct sets of variables in a causal diagram $\G$. Then the following propositions are equivalent:
    \begin{itemize}
        \item $\mathbf Z$ satisfies the \emph{adjustment criterion} relative to $\mathbf{X}$ and $\mathbf{Y}$ in $\G$
        
        \item For all $P$ Markovian with respect to $\G$, we have 
        \[
    P(\mathbf{y} \mid \Do(\mathbf{x})) =
    \begin{cases}
    P(\mathbf{y} \mid \mathbf{x}) & \text{if } \mathbf{Z} = \emptyset, \\
    \sum_{\mathbf{z}} P(\mathbf{y} \mid \mathbf{x}, \mathbf{z}) P(\mathbf{z}) & \text{otherwise.}
    \end{cases}
    \]
    \end{itemize}
\end{theorem}

\section{Identifiability in a Collection of Graphs.}
\label{sec:collectionGraphs} 

When a single causal diagram is considered, do-calculus provides a complete characterization of causal effect identification: a query is identifiable if and only if it can be transformed into a do-free expression via a valid causal derivation in the graph. However, specifying a full causal diagram demands detailed knowledge of all causal relationships among observed variables, an assumption that is often unrealistic in practice.  To mitigate this challenge, many recent works have adopted \emph{causal abstractions}: simplified or coarse-grained models that capture only a subset of causal dependencies.

All the causal abstractions we are aware of determine a collection $\C$ of causal diagrams on $\V$ that satisfy their specified constraints.  We study identifiability of a causal query within such a class: broadly, a query is \emph{identifiable} in $\C$ if the same identification functional can be derived for every graph $G \in \C$.  In practice, however, this notion can be ambiguous.  The literature has proposed multiple identification strategies under causal abstractions, giving rise to several distinct definitions of identifiability in this setting.

\subsection{Identifiability Under Different Knowledge Conditions}

We extend Pearl’s notion of identifiability from a single causal diagram to classes of diagrams. In the classical setting, identifiability is a uniqueness property over a fixed causal graph: a causal query is identifiable if all SCMs inducing the same graph and observational distribution agree on its value. When moving to classes of graphs, this requirement naturally lifts to uniform validity over all SCMs whose graph belongs to the class. Definition~\ref{def:IG} formalizes this extension.

\begin{definition}[Identifiability through Graphs (IG)]
\label{def:IG}
    Let $\C$ be a class of causal diagrams over~$\V$, and let $\X,\Y,\Z\subseteq\V$ be pairwise disjoint with \(\mathbf{Y} \neq \emptyset\). The causal query 
    \(P\bigl(\mathbf{y} \mid \Do(\mathbf{x}),\mathbf{z}\bigr)\)
    is \emph{identifiable through graphs} in $\C$ if there exists an identification functional valid for \(P\bigl(\mathbf{y} \mid \Do(\mathbf{x}),\mathbf{z}\bigr)\) in every SCM $\M$ satisfying \(\G_{\M} \in \C\).
\end{definition}

\begin{example}
    Consider the graphs $\G_1 \coloneqq Z \leftarrow X \rightarrow Y$ and $\G_2 \coloneqq X \leftarrow Z \rightarrow Y$. Then $P(y \mid \Do(x))$ is identifiable through graphs in $\C \coloneqq \{ \G_1, \G_2 \}$. Details are provided in the proof of Proposition~\ref{prop:ICA_not_ICD}.
\end{example}

This notion is called graphical because identifiability is required to hold uniformly over all SCMs whose causal graph lies in $\C$, regardless of their observational distributions. For instance, \citet{jaber_causal_2024} provide a complete characterization of identifiability through graphs in Partial Ancestral Graphs (PAGs), which represent Markov equivalence classes \(\C\) of ADMGs and encode the ancestral relations that are invariant across the class.

We now refine this perspective by incorporating additional observational information. In particular, we consider the case where the true observational distribution $P^\star(\V)$ is known. This induces a further restriction on the admissible SCMs, yielding a smaller class defined by both graphical and distributional constraints.

\begin{definition}[Identifiability through Graphs knowing $P^\star$ (IGP)]
\label{def:IGP}
    Let $\C$ be a class of causal diagrams over~$\V$, and let $\X,\Y,\Z\subseteq\V$ be pairwise disjoint with \(\mathbf{Y} \neq \emptyset\). Let $P^\star(\V)$ denote the true data density. The causal query 
    \(P\bigl(\mathbf{y} \mid \Do(\mathbf{x}),\mathbf{z}\bigr)\)
    is \emph{identifiable through graphs knowing $P^\star$} in $\C$ if there exists an identification functional valid for \(P\bigl(\mathbf{y} \mid \Do(\mathbf{x}),\mathbf{z}\bigr)\) in every SCM $\M$ satisfying \(\G_{\M} \in \C\) and $P_M(\V) = P^\star(\V)$.
\end{definition}

The distinction between IG and IGP lies in the information available to the analyst. IG assumes access only to graphical knowledge, and requires the query to be identifiable uniformly over all SCMs whose causal graph belongs to a given class. IGP additionally assumes knowledge of the true observational distribution, thereby imposing both graphical and distributional constraints, but still allowing for a collection of SCMs. In this sense, IGP corresponds to a strictly more informed setting, in which identifiability is no longer based on graphical information alone. 

Here, IGP should be viewed as a representative example of a broader family of extensions in which identifiability is defined under additional non-graphical information. Such side information may include structural assumptions, such as linearity, additive noise, or restrictions on the functional form of causal mechanisms, as well as distributional assumptions. While other forms of side information could be considered, we focus on IGP as an illustrative example to demonstrate how the framework accommodates such extensions.

From a computational perspective, both notions remain difficult to verify. IG would require validity over all graphs belonging to $\C$, while IGP depends on exact knowledge of $P^\star(\V)$, which is usually unavailable in practice. These limitations motivate our focus on graphical characterizations and sufficient conditions for identifiability in $\C$.

\subsection{Identifiability Under Different Methodological Conditions}

IG provides a natural extension of identifiability to collections of diagrams, without requiring structural assumptions or additional constraints on the observational distribution. A key distinction between notions of identifiability in graphs therefore arises in how identifiability is established over a collection of causal diagrams. As discussed in Section~\ref{sec:singleGraph}, methods developed for a single graph can often be extended to a collection by requiring that the same identification strategy applies uniformly across all compatible graphs. This leads to weaker, yet more tractable, notions of identifiability, trading some degree of generality for computational and conceptual simplicity.

One natural approach consists in applying do-calculus uniformly over the entire collection. More precisely, the objective is to derive a single causal derivation that holds simultaneously in every causal diagram of the class.

\begin{definition}[Identifiability by Common Derivation (ICD)]
\label{def:ICD}
    Let \(\mathcal{C}\) be a class of causal diagrams over the variable set \(\mathcal{V}\), and let \(\mathbf{X}, \mathbf{Y}, \mathbf{Z} \subseteq \mathcal{V}\) be pairwise disjoint with \(\mathbf{Y} \neq \emptyset\). The causal query \(P\bigl(\mathbf{y}\mid \Do(\mathbf{x}),\,\mathbf{z}\bigr)\) is \emph{identifiable by common derivation} in \(\mathcal{C}\) if there exists a causal derivation that holds in every \(\G \in \mathcal{C}\) and transforms the query into an expression involving only observational distributions.
\end{definition}

Moreover, we restate the notion of an \emph{atomically complete calculus}, introduced in \citet{jaber_causal_2024}, for a class \(\C\). Such a calculus consists of three atomic rules $(\mathcal A1, \mathcal A2, \mathcal A3)$, which are counterparts of the three do-calculus rules $(\mathcal R1, \mathcal R2, \mathcal R3)$. Each atomic rule specifies, based only on the structure of the class \(\C\), when the corresponding do-calculus rule application holds uniformly over \(\C\). The rules are required to satisfy the following two properties, for each \(i\):
\begin{enumerate}
  \item \textbf{(Soundness)} If the atomic rule \(\mathcal Ai\) applies to \(\C\), then the corresponding rule \(\mathcal Ri\) holds in every graph in \(\C\).
  \item \textbf{(Completeness)} If the atomic rule \(\mathcal Ai\) does \emph{not} apply to \(\C\), then there exists at least one graph \(G \in \C\) in which the corresponding do-calculus rule \(\mathcal Ri\) does not hold.
\end{enumerate}

Thus, the atomic rules characterize exactly which do-calculus transformations hold uniformly over the class: soundness guarantees that every applicable atomic rule holds in every graph in \(\C\), while completeness guarantees that a transformation cannot hold uniformly over \(\C\) when its corresponding atomic rule does not apply. Consequently, any derivation obtained within an atomically complete calculus immediately yields a causal derivation that holds uniformly over the entire class, thereby establishing ICD for the causal query in \(\mathcal C\). An important example is given by Partial Ancestral Graphs (PAGs), which represent Markov equivalence classes \(\C\) of ADMGs and encode the ancestral relations that are invariant across the class. For this setting, ICD was established in \citet{jaber_causal_2024}. Other examples of atomically complete calculi are given in \citet{anand_causal_2023, yvernes2025relaxing, Ferreira_Assaad_2025}.

The main distinction between identifiability through graphs and identifiability by common derivation lies in the permissible proof techniques. In the latter, one must provide a single causal derivation that holds in all compatible graphs. By contrast, IG permits the use of different graphical identification procedures for different graphs in the class, provided that all resulting procedures ultimately yield the same valid identification functional. 

Beyond do-calculus, graphical identification criteria developed for single graphs can also be extended uniformly to collections of graphs. This leads to notions such as \emph{identifiability by common backdoor} (ICB),  \emph{identifiability by common frontdoor} (ICF) or \emph{identifiability by common adjustment} (ICA)  where one requires that the corresponding graphical criterion is satisfied in every graph of the class $\mathcal C$. More generally, this gives rise to a broad family of graphical identifiability notions based on uniform graphical conditions. For instance, Example~\ref{ex:ICB} illustrates a class in which a causal query is identifiable by common backdoor adjustment.

\begin{example}[ICB]
\label{ex:ICB}
Let us consider the following class of causal diagrams $\C$ as follows:
\[\C \coloneqq 
\left\{ 
\begin{tikzpicture}[ ->, >=stealth, circle, baseline=(current bounding box.center), scale=.8, transform shape]
    \node (X) at (0, 0) {X};
    \node (Y) at (1, 0) {Y};
    \node (Z) at (-1, 0) {Z};
    \node (U) at (-1.707, -0.707) {U};
    \node (V) at (-1.707, 0.707) {V};

    \draw (X) -- (Y);
    \draw (Z) -- (X);
    \draw (U) -- (Z);
    \draw (V) -- (Z);
    \draw (U) -- (V);
    \path (V) edge[bend left=25] (Y);
\end{tikzpicture}
,
\begin{tikzpicture}[ ->, >=stealth, circle, baseline=(current bounding box.center), scale=.8, transform shape]
    \node (X) at (0, 0) {X};
    \node (Y) at (-1, 0) {Y};
    \node (Z) at (1, 0) {Z};
    \node (U) at (1.707, -0.707) {U};
    \node (V) at (1.707, 0.707) {V};

    \draw (X) -- (Y);
    \draw (Z) -- (X);
    \draw (U) -- (Z);
    \draw (V) -- (Z);
    \draw (V) -- (U);
    \path (U) edge[bend left=25] (Y);
\end{tikzpicture}
\right\}
\]

In both graphs of $\C$, $\{Z\}$ is a valid backdoor set. As a result, any causal query $P\left(y \mid \Do(x)\right)$ is identified by the formula $P\left(y \mid \Do(x)\right) = \sum_z P\left(y \mid x, z\right) P(z)$.

\end{example}

As an example, identifiability by common backdoor has been studied in the context of Summary Causal Graphs (SCGs), a causal abstraction that summarizes causal relationships between time series by abstracting away the delays between causes and their effects, in \citet{10.5555/3702676.3702683}. This was later extended into a complete criterion for identifiability by common adjustment by \citet{multivariateICA}. In a different framework, \citet{Perkovic_2018} established a sound and complete criterion for identifiability by common adjustment in Markov equivalence classes of ancestral graphs.

\section{The Strict Hierarchy of Identifiability in Collections of Graphs}
\label{sec:strict_hierarchy}

Section~\ref{sec:collectionGraphs} introduced several notions of identifiability for collections of causal diagrams. These notions are not independent, however, but instead form a hierarchy: some forms of identifiability imply others in all cases. This hierarchy is strict, with the implication relations and counterexamples establishing its strictness summarized in Figure~\ref{fig:liens_notions_id}. We first establish the implication relations in Section~\ref{subsec:implications}, and then show that their converses fail in general and discuss the resulting cases of incomparability in Section~\ref{subsec:strictness}.

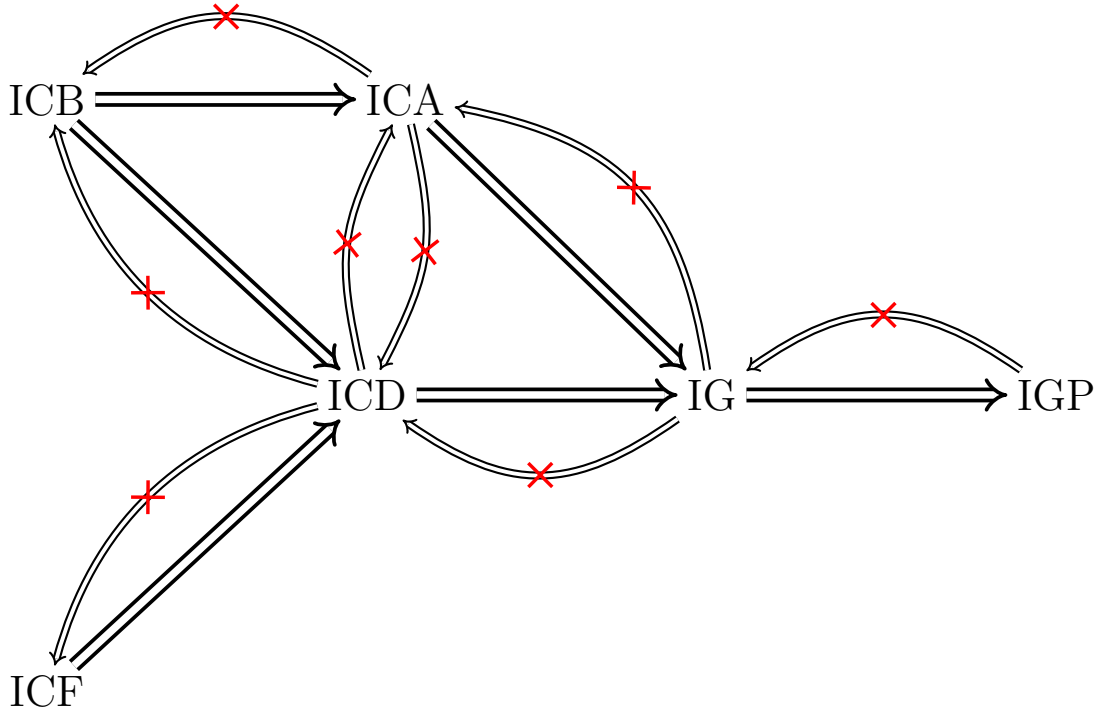
\begin{figure}[t]
    \centering
    \begin{tikzpicture}[  
        implies/.style={-Implies, double, double distance=2.8pt, line width = 1.4pt},
        notimplies/.style={
            -Implies, double, double distance=1.45pt, line width = .8pt,
            postaction={
              decorate,
              decoration={
                markings,
                mark=at position 0.5 with {
                  \draw[-,line width=1.3pt, red] (-4.5pt,-4.5pt) -- (4.5pt,4.5pt);
                  \draw[-,line width=1.3pt, red] (-4.5pt,4.5pt) -- (4.5pt,-4.5pt);
                }
              }
            }
        },
        inner sep=3pt,
        node distance = 3cm,
        scale=1.15, transform shape]

    \node (dots)  {};
    \node (ICD) [right=of dots] {\Large ICD};
    \node (IG)  [right=of ICD]  {\Large IG};
    \node (IGP) [right=of IG]   {\Large IGP};
    \node (ICF) [below=of dots] {\Large ICF};
    \node (ICB) [above=of dots] {\Large ICB};
    \node (ICA) [right =of ICB] {\Large ICA};
    
    \draw[implies] (IG) -- (IGP);
    \draw[implies] (ICD) -- (IG);
    \draw[implies] (ICB) -- (ICA);
    \path (ICA) edge[notimplies, bend right=35,looseness = 1.2] (ICB);
    \path (ICA) edge[implies] (IG);
    \path (ICA) edge[notimplies, bend left=20,looseness = 1.2] (ICD);
    \path (ICB) edge[implies] (ICD);
    \path (ICF) edge[implies] (ICD);
    \path (IGP) edge[notimplies, bend right = 35, looseness = 1.2] (IG);
    \path (ICD) edge[notimplies, bend left =20,looseness = 1.2] (ICA);
    \path (ICD) edge[notimplies, bend left = 30] (ICB);
    \path (ICD) edge[notimplies, bend right =30] (ICF);
    \path (IG) edge[notimplies, bend right=35,looseness = 1.2] (ICA);
    \path (IG) edge[notimplies, bend left =35, looseness = 1.2] (ICD);
    \end{tikzpicture}
    \caption{Hierarchy between the notions of identifiability. 
    An arrow indicates that the implication holds in every case. 
    A red cross on an arrow indicates that the implication fails in at least one case. 
    }

    \label{fig:liens_notions_id}
\end{figure}

\subsection{The Hierarchy of Identifiability}
\label{subsec:implications}

We now establish the implications underlying the hierarchy. We first consider the relationship between IG and IGP. Since IGP corresponds to IG with additional knowledge, we know that IG always implies IGP, as formalized by Proposition~\ref{prop:IG_implique_IGP}.

\begin{restatable}[$IG \Rightarrow IGP$]{proposition}{propIGimpliqueIG}
\label{prop:IG_implique_IGP}
    Let \(\mathcal{C}\) be a class of causal diagrams over the variable set \(\mathcal{V}\), let $P^\star$ be a distribution over $\mathcal V$, and let \(\mathbf{X}, \mathbf{Y}, \mathbf{Z} \subseteq \mathcal{V}\) be pairwise disjoint with \(\mathbf{Y} \neq \emptyset\). If the causal query \(P\bigl(\mathbf{y}\mid \Do(\mathbf{x}),\,\mathbf{z}\bigr)\) is identifiable through graphs in \(\mathcal{C}\), then it is identifiable through graphs knowing $P^\star$ in $\C$.
\end{restatable}

\begin{proof}
    Assume that IG holds for
    $
    P(\mathbf{y}\mid \Do(\mathbf{x}), \mathbf z)
    $
    over $\mathcal{C}$.
    Then there exists an identification functional $\phi$ such that for every SCM $\mathcal{M}$ satisfying
    $
    \mathcal{G}_{\mathcal{M}}\in\mathcal{C},
    $
    \(
        P_{\mathcal{M}}(\mathbf{y}\mid \Do(\mathbf{x}),\mathbf z)
        =
        \phi(P_{\mathcal{M}}).
    \)
    Since $\{\mathcal M \mid \G_{\mathcal M} \in \mathcal C, P_{\mathcal M} = P^\star\} \subseteq \{\mathcal M \mid \G_{\mathcal M} \in \mathcal C\}$, $\phi$ is valid for all SCM $\mathcal M$ such that $\G_{\mathcal M} \in \mathcal C$ and such that $P_{\mathcal M} = P^\star$.
    Hence IGP holds.
\end{proof}

Similarly, because ICD requires a single causal derivation to hold simultaneously across all causal diagrams in the class, ICD implies IG, as stated by Proposition~\ref{prop:ICD_implique_IG}.

\begin{proposition}[$ICD \Rightarrow IG$]
\label{prop:ICD_implique_IG}
    Let \(\mathcal{C}\) be a class of causal diagrams over the variable set \(\mathcal{V}\), and let \(\mathbf{X}, \mathbf{Y}, \mathbf{Z} \subseteq \mathcal{V}\) be pairwise disjoint with \(\mathbf{Y} \neq \emptyset\). If the causal query \(P\bigl(\mathbf{y}\mid \Do(\mathbf{x}),\,\mathbf{z}\bigr)\) is identifiable by common derivation in \(\mathcal{C}\), then it is identifiable through graphs in $\C$.
\end{proposition}

\begin{proof}
    Assume that ICD holds for $P(\mathbf{y}\mid \Do(\mathbf{x}),\mathbf z)$ over $\mathcal{C}$. By definition, there exists a single do-calculus derivation that holds for every graph in $\mathcal{C}$ and produces an identification formula
    $
    \mathcal{F}(P).
    $
    Consequently, for every SCM $\mathcal{M}$ such that
    $
    \mathcal{G}_{\mathcal{M}}\in\mathcal{C},
    $
    we have
    \(
        P_{\mathcal{M}}(\mathbf{y}\mid \Do(\mathbf{x}),\mathbf z)
        =
        \mathcal{F}(P_{\mathcal{M}}).
    \)
    Therefore, the identification functional $(P_{\mathcal{M}} \mapsto \mathcal{F}(P_{\mathcal{M}}))$ is valid for all $\mathcal M$ such that $\G_{\mathcal M} \in \mathcal C$. Hence IG holds.
\end{proof}

Similarly, methods that consist of applying a graphical criterion uniformly across all graphs in the collection (i.e., ICB, ICA, and ICF) imply either IG or ICD. Specifically, ICB and ICF imply ICD, while ICA implies IG. Theorem~\ref{th:liens_notions_id} summarizes all implication relations between the identifiability notions introduced above.

\begin{restatable}[Hierarchy of Identifiability Notions]{theorem}{theoremhierarchy}
\label{th:liens_notions_id}
    Let $\mathcal{C}$ be a class of causal diagrams over $\mathcal{V}$, and let $\mathbf{X},\mathbf{Y},\mathbf{Z}\subseteq\mathcal{V}$ be pairwise disjoint with \(\mathbf{Y} \neq \emptyset\). 
    The implications between identifiability notions summarized in Figure~\ref{fig:liens_notions_id} are correct.
\end{restatable}

\begin{proof}[Proof Sketch]
    The implications $IG \Rightarrow IGP$ and $ICD \Rightarrow IG$ follow directly from
    Propositions~\ref{prop:IG_implique_IGP} and~\ref{prop:ICD_implique_IG},
    respectively. The implications $ICB \Rightarrow ICD$ and $ICF \Rightarrow ICD$
    follow from the fact that backdoor and frontdoor adjustment admit a common
    do-calculus derivation~\citep{pearl_causality_2009}, which can therefore be
    applied uniformly across all graphs in $\mathcal{C}$. Finally, $ICB \Rightarrow ICA$
    follows from the fact that the backdoor criterion is a special case of
    Definition~\ref{def:critere_ajustement}, while $ICA \Rightarrow IG$ follows from
    Theorem~\ref{th:adjustment_perkovic}. Full details are given in Appendix~\ref{sec:proof:prop}.
\end{proof}

\subsection{Strictness of the Hierarchy}
\label{subsec:strictness}

The hierarchy presented in Theorem~\ref{th:liens_notions_id} is strict. More specifically, the following result holds.

\begin{restatable}[Strictness of the Hierarchy]{proposition}{propstrictness}
\label{prop:strictness}
    None of the converse implications in Figure~\ref{fig:liens_notions_id} holds in general, meaning that there are counterexamples to each of them. Moreover, ICA and ICD are incomparable, meaning that there are counterexamples in both directions.
\end{restatable}

We establish Proposition~\ref{prop:strictness} by considering the relevant separations individually. Most of them follow from existing counterexamples in the single-graph setting or from straightforward constructions. The separation between IG and ICD, however, requires a substantially more involved construction and constitutes the main technical contribution of this section. We first establish the simpler separations and then turn to this more challenging case. The formal proof, including all technical details, is provided in Appendix~\ref{proof:strictness}.

\subsubsection{Straightforward Separations}
\label{subsubsec:simple_sep}

We first establish the separations that follow directly from known results in the single-graph setting or from simple constructions.
In the single-graph setting, the back-door criterion, the front-door criterion, and generalized adjustment are not complete for causal effect identification, whereas do-calculus is complete \citep{pearl_causality_2009, Perkovic_2015}. These limitations immediately carry over to collections of causal diagrams, since the single-graph setting is recovered by considering collections consisting of a single diagram. Consequently, the corresponding converse implications in Figure~\ref{fig:liens_notions_id} do not hold in general. The formal constructions are given in Appendix~\ref{proof:strictness}; they can also be obtained directly from the standard counterexamples for the incompleteness of these criteria.

We next consider the converse implication \(IGP \Rightarrow IG\). Since IGP incorporates additional information about the data-generating distribution, it can identify effects that are not identifiable from the graphical structure alone. 

\begin{restatable}[Failure of \(IGP \Rightarrow IG\)]{example}{failureIGPfIG}
\label{ex:failureIGPfIG}
    Consider the collection of causal diagrams \(\C \coloneqq \{X \rightarrow Y\,; X \leftarrow Y\}.\)
    Suppose that the observed variables \(X\) and \(Y\) are independent.
    Then \(P(y\mid \Do(x))\) is identifiable from \(\C\) knowing the data density \(P^\star\), although it is not identifiable from \(\C\) alone. Thus, IGP does not imply IG. All details are provided in Appendix~\ref{sec:proof:simple_sep}.
\end{restatable}

The example is deliberately simple, as our definitions allow the observed distribution \(P^\star\) to provide information beyond the conditional independences encoded by the graph. More elaborate instances of the same phenomenon can be obtained by imposing additional distributional constraints. In that direction, \citet{tikka2019identifying} show that context-specific independence relations can provide information beyond the conditional independences encoded by a causal graph, and can turn an otherwise non-identifiable causal effect into an identifiable one. Thus, this separation between graphical and distributional information already occurs in the single-graph setting. The example above illustrates the same underlying mechanism in its simplest form: information contained in \(P^\star\) can rule out causal models that remain compatible with the graphical information alone.

\subsubsection{ICA, ICD and IG}
\label{subsubsec:ICDvsIG}

Historically, the relationship between ICD and IG has been an open question, particularly in the context of causal abstractions such as ancestral graphs PAGs. \citet{JMLR:v9:zhang08a} first extended do-calculus to ancestral graphs, providing a framework for identifiability by common derivation within such classes. However, their work left unresolved whether IG could be strictly more general than ICD. This gap was later addressed by \citet{jaber_causal_2024}, who proved that in the specific case of PAGs, ICD and IG are in fact equivalent. Their result established that, for PAGs, the existence of a common causal derivation is both necessary and sufficient for identifiability through graphs.

Yet, this equivalence does not hold in full generality. In this section, we demonstrate that  IG and ICD are distinct notions. We begin by showing the failure of the implication from ICA to ICD, thereby illustrating that this graphical criterion do not guarantee a common derivation. To formalize this separation, we introduce derivation graphs from \citet{yvernes2026unveiling}, as a tool to characterize the structure of possible causal derivations.

\begin{definition}[Derivation Graph \citep{yvernes2026unveiling}]
    Let $\mathcal{G} = (\mathcal{V}_\mathcal{G}, \mathcal{E}_\mathcal{G})$ be an ADMG. 
    Its \emph{derivation graph} is defined as the undirected graph $\mathcal{D}[\mathcal{G}] = (\mathcal{V}_\mathcal{D}, \mathcal{E}_\mathcal{D})$, where:
    \begin{itemize}
        \item The set of vertices $\mathcal{V}_\mathcal{D}$ consists of all expressions of the form $P(\mathbf{y} \mid \Do(\mathbf{x}), \mathbf{w})$, with pairwise disjoint subsets  $(\mathbf{Y}, \mathbf{X},  \mathbf{W}) \subseteq \mathcal{V}_\mathcal{G}$, with $\mathbf{Y} \neq \emptyset$.
        \item There is an edge $(Q_1, Q_2) \in \mathcal{E}_\mathcal{D}$ if $Q_2$ can be obtained from $Q_1$ by applying a single do-calculus rule  that is valid in $\mathcal{G}$.
    \end{itemize}
\end{definition}

\begin{figure}[t]
    \centering
    \begin{tikzpicture}[
        dofree/.style = {
          draw=coolorange!60!black,
          fill=coolorange!80,
          line width=0.6pt,
          inner sep=1pt,
          font=\scriptsize
        },
        withdo/.style = {
          draw=blue!60!black,
          fill=blue!25,
          line width=0.6pt,
          inner sep=1pt,
          font=\scriptsize
        },
        styleR1/.style = {thick, -, gray, line width=1.2pt},
        styleR2/.style = {thick, -, orange!60!brown, dashed, line width=1.2pt},
        styleR3/.style = {thick, -, blue!55, dotted, line width=1.5pt},
        scale=3
    ]
    
    \node[dofree] (n0) at (0.000, 0.475) {$\mathrm{P(x \mid z)}$};
    \node[withdo] (n1) at (0.000, -.475) {$\mathrm{P(x \mid \Do(y,z))}$};
    \node[withdo] (n2) at (-.475, 0.000) {$\mathrm{P(x \mid \Do(y), z)}$};
    \node[withdo] (n3) at (0.475, 0.000) {$\mathrm{P(x \mid \Do(z))}$};
    
    \node[withdo] (n4) at (2.175, 0.000) {$\mathrm{P(y \mid \Do(x), z)}$};
    \node[withdo] (n5) at (1.700, -.475) {$\mathrm{P(y \mid \Do(x,z))}$};
    \node[withdo] (n6) at (1.225, 0.000) {$\mathrm{P(y \mid \Do(z), x)}$};
    \node[dofree] (n7) at (1.700, 0.475) {$\mathrm{P(y \mid x,z)}$};
    
    \node[withdo] (n8)  at (3.875, 0.000) {$\mathrm{P(z \mid \Do(y))}$};
    \node[withdo] (n9)  at (2.925, 0.000){$\mathrm{P(z \mid \Do(x))}$};
    \node[withdo] (n10) at (3.400, -.475) {$\mathrm{P(z \mid \Do(x), y)}$};
    \node[dofree] (n11) at (3.400, 0.475) {$\mathrm{P(z)}$};
    
    \node[dofree] (n12) at (-0.350,-1.1) {$\mathrm{P(x)}$};
    \node[withdo] (n13) at (0.350,-1.1) {$\mathrm{P(x \mid \Do(y))}$};
    
    \node[withdo] (n14) at (1.350,-1.1) {$\mathrm{P(x \mid \Do(z), y)}$};
    \node[dofree] (n15) at (2.050,-1.1) {$\mathrm{P(x \mid y,z)}$};
    
    \node[withdo] (n16) at (3.050,-1.1) {$\mathrm{P(y \mid \Do(z))}$};
    \node[dofree] (n17) at (3.750,-1.1) {$\mathrm{P(y \mid z)}$};
    
    \node[dofree] (n18) at (-0.350, -1.775) {$\mathrm{P(z \mid x)}$};
    \node[withdo] (n19) at (0.350, -1.775) {$\mathrm{P(z \mid \Do(y), x)}$};
    
    \node[withdo] (n20) at (1.350, -1.775) {$\mathrm{P(x,y \mid \Do(z))}$};
    \node[dofree] (n21) at (2.050, -1.775) {$\mathrm{P(x,y \mid z)}$};
    
    \node[withdo] (n22) at (3.050, -1.775) {$\mathrm{P(x,z \mid \Do(y))}$};
    \node[dofree] (n23) at (3.750, -1.775) {$\mathrm{P(x,z)}$};
    
    
    \draw[styleR2] (n1) -- (n2);
    \draw[styleR3] (n1) -- (n3);
    \draw[styleR3] (n2) -- (n0);
    \draw[styleR3] (n13) -- (n12);
    \draw[styleR2] (n14) -- (n15);
    \draw[styleR2] (n3) -- (n0);
    \draw[styleR2] (n5) -- (n6);
    \draw[styleR2] (n5) -- (n4);
    \draw[styleR2] (n4) -- (n7);
    \draw[styleR2] (n6) -- (n7);
    \draw[styleR2] (n16) -- (n17);
    \draw[styleR3] (n10) -- (n8);
    \draw[styleR3] (n10) -- (n9);
    \draw[styleR3] (n9) -- (n11);
    \draw[styleR3] (n19) -- (n18);
    \draw[styleR3] (n8) -- (n11);
    \draw[styleR2] (n20) -- (n21);
    \draw[styleR3] (n22) -- (n23);
    \end{tikzpicture}
    \caption{Derivation graph of the DAG from Example~\ref{ex:identification_avec_derivation}.
    Blue nodes denote interventional queries, while orange nodes denote do-free queries.
    Edges of the derivation graph represent valid applications of the do-calculus rules in the DAG: gray 
    solid lines for Rule $\mathcal{R}1$, orange dashed lines
    for Rule $\mathcal{R}2$, and blue dotted lines for Rule $\mathcal{R}3$. Only vertices that are non-isolated are displayed.}
    \label{fig:dg2}
\end{figure}
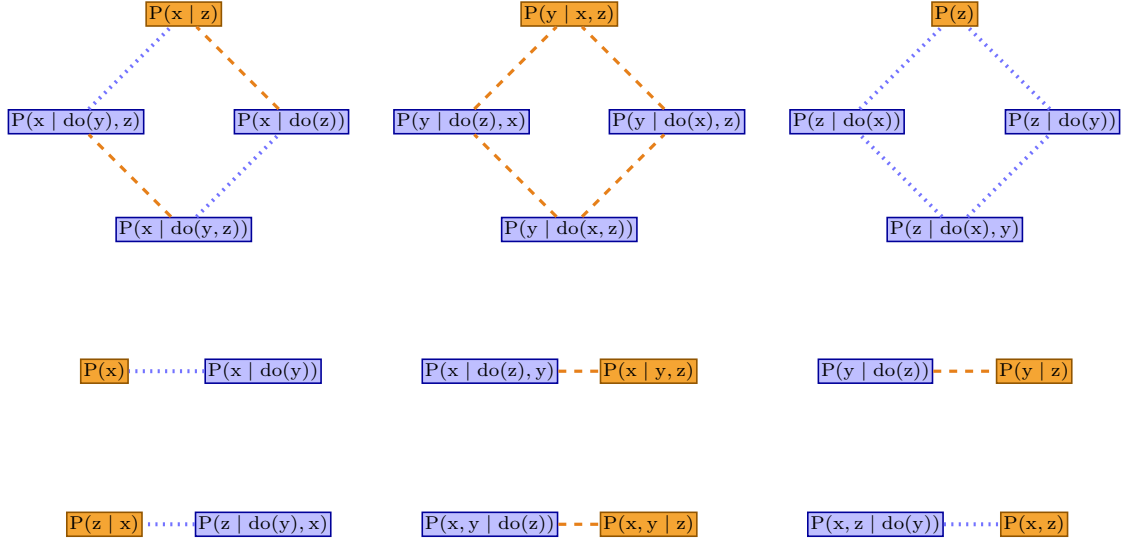

The derivation graph takes the causal diagram as its starting point and characterizes the do-calculus rules that hold in it. Figure~\ref{fig:dg2} provides an example of a derivation graph, illustrating how the graphical structure determines which transformations are permitted. 

We now reverse this perspective: rather than starting from a graph and deriving the rules that hold, we start from a given causal derivation and ask which graphical conditions are required for each of its individual steps to hold. This shift in perspective allows us to characterize the graphical conditions associated with each step of a derivation and leads to the notion of a \emph{step graph}.

\begin{definition}[Step Graph of a Derivation]
\label{def:step_graph}
    Let $D$ be a causal derivation. Let $\V$ be the set of all variables represented in $D$. Its \emph{step graph} is defined as the following undirected graph $\mathcal{S}[D] = (\mathcal{V}_\mathcal{S}, \mathcal{E}_\mathcal{S})$, where:
    \begin{itemize}
        \item The set of vertices $\mathcal{V}_\mathcal{D}$ consists of all expressions of the form $P(\mathbf{y} \mid \Do(\mathbf{x}), \mathbf{w})$, with pairwise disjoint subsets  $(\mathbf{Y}, \mathbf{X},  \mathbf{W}) \subseteq \mathcal{V}$, with $\mathbf{Y} \neq \emptyset$, that appear in $D$.
        \item There is an edge $(Q_1, Q_2) \in \mathcal{E}_\mathcal{S}$ if $Q_1$ is rewritten into $Q_2$ by applying a single do-calculus rule in $D$. 
    \end{itemize}
\end{definition}

\begin{example}
\label{ex:step_graph}
    Consider the derivation from Example~\ref{ex:identification_avec_derivation}. This derivation involves the variables $X$, $Y$, and $Z$, and uses $\mathcal{R}3$ to obtain
    $P(z \mid \Do(x)) = P(z)$ and $\mathcal{R}2$ to obtain
    $P(y \mid \Do(x), z) = P(y \mid x, z)$.
    As a result, the nodes of its step graph consist of all causal queries that can be written using ${X,Y,Z}$, and its edges are
    $P(z \mid \Do(x)) \undirectedarrow P(z)$ and
    $P(y \mid \Do(x), z) \undirectedarrow P(y \mid x, z)$.
\end{example}

With these definitions, derivation graphs can be used to characterize the set of valid derivations in any graph $G \in \mathcal{C}$. Conversely, step graphs can be used to characterize the set of graphs in which a given derivation is valid. This duality is formalized in the following theorem.

\begin{theorem}
    \label{thm:derivation_validity}
    Let $\V$ be a set of variables, $\G$ be an ADMG on $\mathcal V$, and $D$ be a derivation on $\mathcal V$. The following holds:
    \[
    D \text{ holds in } \G \quad \text{if and only if} \quad \mathcal{S}[D] \subseteq \mathcal{D}[\G].
    \]
\end{theorem}

\begin{proof}
    By construction, $\mathcal{S}[D]$ contains all expressions and do-calculus rules required for the derivation $D$. If $D$ holds in $\G$, then every rule in $D$ is valid in $\G$, hence all edges of $\mathcal{S}[D]$ are present in $\mathcal{D}[\G]$, implying $\mathcal{S}[D] \subseteq \mathcal{D}[\G]$. Conversely, if $\mathcal{S}[D] \subseteq \mathcal{D}[\G]$, then all rules in $D$ are valid in $\G$, hence $D$ holds in $\G$.
\end{proof}

\begin{example}
Let $\G$ be the DAG considered in Example~\ref{ex:identification_avec_derivation}, with derivation graph $\mathcal{D}[\G]$ shown in Figure~\ref{fig:dg2}. Let $D$ be the derivation from the same example. $D$ holds in $\G$, and its step graph $\mathcal{S}[D]$, constructed in Example~\ref{ex:step_graph}, satisfies $$\mathcal{S}[D] \subseteq \mathcal{D}[\G].$$ 
Conversely, consider a derivation $D'$ that applies $\mathcal{R}3$ to obtain
$P(x \mid \Do(z)) = P(x)$.
This rule is not valid in $\G$. Consequently the edge
$P(x \mid \Do(z)) \undirectedarrow P(x)$ is present in the step graph but not in the derivation graph. Hence, $D'$ does not hold in $\G$ and $$\mathcal{S}[D'] \not\subseteq \mathcal{D}[\G].$$
\end{example}

Building on this theorem, we demonstrate the failure of the implication from ICA to ICD.

\begin{restatable}[Failure of ICA $\Rightarrow$ ICD]{proposition}{propICAnotICD}
\label{prop:ICA_not_ICD}
    There exist a causal query $P(\mathbf{y} \mid \Do(\mathbf{x}), \mathbf{z})$ and a class of causal graphs $\mathcal{C}$ such that $P(\mathbf{y} \mid \Do(\mathbf{x}), \mathbf{z})$ is identifiable by common adjustment in $\mathcal{C}$, but not identifiable by common derivation in $\mathcal{C}$.
\end{restatable}

\begin{proof}[Proof Sketch]
    See Appendix~\ref{sec:proof:ICDvsIG} for a detailed proof. We construct a counterexample in which ICA holds but ICD does not.
    Consider the graphs $\G_1 \coloneqq Z \leftarrow X \rightarrow Y$ and
    $\G_2 \coloneqq X \leftarrow Z \rightarrow Y$, and the class $\C \coloneqq \{ \G_1, \G_2 \}$. We have the following properties:
    \begin{itemize}
        \item $\{Z\}$ satisfies the adjustment criterion in both $\G_1$ and $\G_2$, hence ICA holds in $\C$ for $P(y \mid \Do(x))$.

        \item Consider the graph $\G_3 \coloneqq X \rightarrow Z \rightarrow Y$. One can verify that:
        \[
        \mathcal{D}[\G_1] \cap \mathcal{D}[\G_2] \subseteq \mathcal{D}[\G_3].
        \]
        Assume, by contradiction, that $P(y \mid \Do(x))$ is ICD in $\C$. Then, there exists a derivation $D$ that identifies $P(y \mid \Do(x))$ in all $\G \in \C$. Hence:
        \[
        \mathcal{S}[D] \subseteq \mathcal{D}[\G_1] \cap \mathcal{D}[\G_2] \subseteq \mathcal{D}[\G_3].
        \]
        Therefore, by Theorem~\ref{thm:derivation_validity}, $D$ holds in $\G_3$, hence $P(y \mid \Do(x))$ is ICD, thus IG (by Proposition~\ref{prop:ICD_implique_IG}), in $\{\G_1, \G_2, \G_3\}$.
        However, since $\{Z\}$ does not satisfy the adjustment criterion in $\G_3$, by Theorem~\ref{th:adjustment_perkovic}, there exists a density $P^\star$ Markovian with respect to $\G_3$ such that:
        \[
        P^\star(y \mid \Do(x)) \neq \sum_{z} P^\star(y\mid z,x)P^\star(z).
        \]
        The construction of $P^\star$ yields an SCM $\mathcal{M}_3$ with $P_{\mathcal{M}_3} = P^\star$ and $\G_{\mathcal{M}_3} = \G_3$. Moreover, $P^\star$ is also Markovian with respect to $\G_2$, hence by Proposition~7.1 from \cite{PetJanSch17}, there exists an SCM $\mathcal{M}_2$ with $P_{\mathcal{M}_2} = P^\star$ and $\G_{\mathcal{M}_2} = \G_2$. In $\G_2$, since $\{Z\}$ satisfies the adjustment criterion, we have:
        \[
        P^\star(y \mid \Do(x)) = \sum_{z} P^\star(y\mid z,x)P^\star(z).
        \]
        \textbf{This contradicts the previous inequality}.
        Therefore, $P(y \mid \Do(x))$ cannot be ICD in $\C$.
    \end{itemize}
    Hence, $ICA \Rightarrow ICD$ fails.
\end{proof}

The failure of ICA $\Rightarrow$ ICD also has direct implications for the relationship between IG and ICD. Since IG includes ICA as a special case, the non-implication of ICA to ICD extends directly to IG. This is formalized in Corollary~\ref{cor:IG_not_ICD}.

\begin{corollary}[Failure of IG $\Rightarrow$ ICD]
\label{cor:IG_not_ICD}
    There exists a class of causal diagrams $\mathcal{C}$ and a causal query $P(\mathbf{y} \mid \Do(\mathbf{x}), \mathbf{z})$ such that  $P(\mathbf{y} \mid \Do(\mathbf{x}), \mathbf{z})$ is identifiable through graphs in $\mathcal{C}$, but \emph{not} identifiable by common derivation in $\mathcal{C}$.
\end{corollary}

\begin{proof}
Let $\mathcal{C}$ be the class constructed in Proposition~\ref{prop:ICA_not_ICD}. Since ICA holds in $\mathcal{C}$, there exists a single identification functional (e.g., the adjustment formula) valid for all SCMs with graphs in $\mathcal{C}$. Thus, IG holds by Definition~\ref{def:IG}. However, as shown in Proposition~\ref{prop:ICA_not_ICD}, ICD fails in $\mathcal{C}$. This establishes the corollary.
\end{proof}

The results above have implications for the use of atomically complete calculi (e.g., \citet{yvernes2025relaxing}) as tools for identifying causal queries in causal abstractions. While an atomically complete calculus guarantees ICD (and thus IG, since ICD implies IG), our findings show that atomic completeness is stronger than necessary for IG. Indeed, IG may hold even when no atomically complete calculus exists, because atomic completeness requires derivations to be uniform across all graphs represented by a causal abstraction, whereas IG allows graph-specific derivations. This suggests that characterizing IG may require identification tools that are more flexible than atomically complete calculi.

\section{Discussion}
\label{sec:discussions}

\subsection{Interpretation of the Hierarchy}

Beyond its formal implications, the hierarchy provides a useful lens for interpreting existing results on causal effect identifiability in abstractions. It makes it possible to distinguish precisely which notion is characterized by a given approach and to identify potential gaps between existing guarantees and stronger notions of identifiability. In particular, a claim of completeness should always be interpreted with respect to a specific level of the hierarchy: completeness for ICD, for instance, does not imply completeness for IG. The atomically complete calculus established in \citet{yvernes2025relaxing} solves ICD for the causal abstractions and causal queries considered in this work, while IG remains open. Thus, the hierarchy provides a framework for locating existing results and assessing which stronger notions remain uncharacterized.

From this perspective, IG can be viewed as a natural target when causal effects are to be identified from a purely graphical abstraction. Among the notions considered here, IG is the least restrictive, as it requires only graphical identifiability and does not impose a specific form of identification strategy. When the available information is purely graphical, establishing IG guarantees that no effect identifiable from this information is missed. Thus, IG provides a natural completeness target for purely graphical causal abstractions. This makes IG particularly useful for characterizing the limits of what can be inferred from the graphical abstraction alone.

Conversely, if a causal effect is not IG under the available graphical abstraction, then identifying it requires information beyond that encoded by the abstraction. Such additional information may take the form of a refinement of the graphical abstraction or of non-graphical assumptions or information about the underlying causal model. In this sense, failure of IG does not merely indicate that a particular identification method is insufficient; it indicates that the available graphical information alone is insufficient to identify the effect.

\subsection{Maximal Elements Is All You Need}

Identifiability notions defined over a collection of causal diagrams can become difficult to apply when the collection \(\mathcal{C}\) is large. A useful reduction is that, for the notions considered in this paper, it is sufficient to analyze only the maximal elements of \(\mathcal{C}\) under graph inclusion. The key reason for this reduction is monotonicity under edge deletion: if an independence property holds in a graph, then it also holds in any graph obtained from it by deleting edges.

We make the following assumption throughout this section.

\begin{assumption}[Existence of Maximal Elements]
\label{ass:maximal_elements}
    For every graph \(\G\in\mathcal{C}\), there exists a maximal element \(\G^{\max}\in\mathcal{C}\) such that \(\G\) is contained in \(\G^{\max}\).
\end{assumption}

This assumption is automatic when \(\mathcal{C}\) is finite: starting from any \(\G\in\mathcal{C}\), repeatedly moving to a strict supergraph in \(\mathcal{C}\) must eventually terminate at a maximal element containing \(\G\). In particular, when \(\mathcal{V}\) is finite, there are only finitely many causal diagrams over \(\mathcal{V}\), and hence \(\mathcal{C}\) is finite. Finiteness of \(\mathcal{C}\), however, is not necessary. The assumption also holds in some infinite settings, such as causal abstractions for time series, including Extended Summary Causal Graphs (ESCGs) \citet{assaad2022} and Summary Causal Graphs (SCGs) \citet{multivariateICA}.

Lemma~\ref{lemma:Identifiabilite_et_inclusion} formalizes this monotonicity.

\begin{restatable}[Monotonicity under Graph Inclusion]{lemma}{lemmamonotonicity}
\label{lemma:Identifiabilite_et_inclusion}
Let \(\mathcal{V}\) be the set of observable variables, and let \(\G_1\) and \(\G_2\) be two causal diagrams on \(\mathcal{V}\) such that \(\G_2\) is a subgraph of \(\G_1\). Let \(\mathbf{X}, \mathbf{Y}, \mathbf{Z} \subseteq \mathcal{V}\) be pairwise disjoint with \(\mathbf{Y} \neq \emptyset\). Then:
\begin{enumerate}
    \item Any identification functional valid in \(\G_1\) for
    \(P\bigl(\mathbf{y}\mid \Do(\mathbf{x}), \mathbf{z}\bigr)\)
    is also valid in \(\G_2\).
    
    \item Let \(P^\star(\mathcal{V})\) be an observational distribution.
    Any identification functional valid for
    \(P\bigl(\mathbf{y}\mid \Do(\mathbf{x}), \mathbf{z}\bigr)\)
    in all SCMs $\mathcal M_1$ inducing \(\G_1\) and satisfying
    \(P_{\mathcal M_1}(\mathcal{V})=P^\star(\mathcal{V})\) is also valid in all SCMs $\mathcal M_2$
    inducing \(\G_2\) and satisfying
    \(P_{\mathcal M_2}(\mathcal{V})=P^\star(\mathcal{V})\).
    
    \item Any causal derivation that holds in \(\G_1\) also holds in \(\G_2\).
    \item Any graphical criterion satisfied in \(\G_1\) is also satisfied in \(\G_2\).
\end{enumerate}
\end{restatable}

\begin{proof}[Proof sketch]
    All four statements follow from the fact that \(\G_2\) is obtained from \(\G_1\) by removing edges. 
    
    For items~(1) and~(2), any SCM inducing \(\G_2\) can be extended to an SCM inducing \(\G_1\) by adding the missing parent arguments to the structural equations while keeping the functional dependence unchanged on the original parents. This extension preserves both the observational and interventional distributions. In particular, it preserves any fixed observational distribution \(P^\star(\mathcal{V})\). Items~(3) and~(4) follow from the fact that every graphical independence that holds in \(\G_1\) also holds in \(\G_2\). 
    
    The full proof is given in Appendix~\ref{proof:further}.
\end{proof}

Lemma~\ref{lemma:Identifiabilite_et_inclusion} implies that the identifiability notions considered in this paper are determined by the maximal graphs in \(\mathcal{C}\). This is formalized in the following theorem.

\begin{theorem}[Reduction to Maximal Elements]
\label{th:identifiabilite_et_elements_maximaux}
    Let \(\mathcal{C}\) be a class of causal diagrams over a set of variables \(\mathcal{V}\) satisfying Assumption~\ref{ass:maximal_elements}, and let \(\mathbf{X}, \mathbf{Y}, \mathbf{Z} \subseteq \mathcal{V}\) be pairwise disjoint with \(\mathbf{Y} \neq \emptyset\). Let \(\mathcal{C}^{\max}\) denote the subcollection of maximal elements of \(\mathcal{C}\) under graph inclusion. Consider any notion of identifiability discussed above (namely, IG, IGP, ICD, ICB, ICF, or ICA). Then the following statements are equivalent:
    \begin{enumerate}
        \item \(P\bigl(\mathbf{y}\mid \Do(\mathbf{x}), \mathbf{z}\bigr)\) is identifiable in \(\mathcal{C}\).
        \item \(P\bigl(\mathbf{y}\mid \Do(\mathbf{x}), \mathbf{z}\bigr)\) is identifiable in \(\mathcal{C}^{\max}\).
    \end{enumerate}
\end{theorem}

\begin{proof}
The implication \((1)\Rightarrow(2)\) is immediate since \(\mathcal{C}^{\max}\subseteq \mathcal{C}\) by Assumption~\ref{ass:maximal_elements}.

For \((2)\Rightarrow(1)\), let \(\G\in\mathcal{C}\). By Assumption~\ref{ass:maximal_elements}, there exists \(\G^{\max}\in\mathcal{C}^{\max}\) such that \(\G\) is a subgraph of \(\G^{\max}\).

We distinguish the different notions of identifiability considered above:
\begin{itemize}
    \item If identifiability in \(\mathcal{C}^{\max}\) is established through a common identification functional, then this functional is valid in \(\G^{\max}\). By Lemma~\ref{lemma:Identifiabilite_et_inclusion}, it is also valid in \(\G\).
    
    \item If identifiability in \(\mathcal{C}^{\max}\) is established through a common identification functional knowing \(P^\star\), then this functional is valid for all SCMs inducing \(\G^{\max}\) and satisfying \(P_M(\V)=P^\star(\V)\). By Lemma~\ref{lemma:Identifiabilite_et_inclusion}, it is also valid for all SCMs inducing \(\G\) and satisfying \(P_M(\V)=P^\star(\V)\).

    \item If identifiability in \(\mathcal{C}^{\max}\) is established through a common causal derivation, then this derivation holds in \(\G^{\max}\). Again by Lemma~\ref{lemma:Identifiabilite_et_inclusion}, it also holds in \(\G\).

    \item If identifiability in \(\mathcal{C}^{\max}\) is established through a common graphical criterion, then this criterion is satisfied in \(\G^{\max}\). By Lemma~\ref{lemma:Identifiabilite_et_inclusion}, it is therefore also satisfied in \(\G\).
\end{itemize}

Hence, the query is identifiable in \(\mathcal{C}\).
\end{proof}

Theorem~\ref{th:identifiabilite_et_elements_maximaux} shows that every notion of identifiability becomes easier to assess when \(\mathcal{C}\) has few maximal elements. In particular, if \(\mathcal{C}\) admits a greatest element under graph inclusion, then identifiability in \(\mathcal{C}\) reduces to identifiability in a single graph.

\begin{corollary}
\label{cor:single_max}
If \(\mathcal{C}\) contains a greatest element under graph inclusion, then IG is equivalent to ICD.
\end{corollary}

\begin{proof}
    Let \(\G^{\star}\) denote the greatest element of \(\mathcal{C}\). Then
    \(
    \mathcal{C}^{\max} = \{\G^{\star}\}.
    \)
    By Theorem~\ref{th:identifiabilite_et_elements_maximaux}, IG in \(\mathcal{C}\) is equivalent to IG in \(\mathcal{C}^{\max}\), and thus to IG in the single graph \(\G^{\star}\).
    On a single causal diagram, soundness and completeness of do-calculus imply that IG in \(\G^{\star}\) is equivalent to ICD in \(\{\G^{\star}\}\).
    Finally, applying again Theorem~\ref{th:identifiabilite_et_elements_maximaux} yields that ICD in \(\G^{\star}\) is equivalent to ICD in \(\mathcal{C}\). Hence IG and ICD are equivalent in \(\mathcal{C}\).
\end{proof}

This situation arises, for example, in cluster DAGs \citep{anand_causal_2023}, where the partition admissibility assumption ensures the existence of a unique maximal element under inclusion. When this assumption is relaxed, as in \citep{yvernes2025relaxing}, Corollary~\ref{cor:single_max} no longer applies. In that setting, ICD has been established, whereas IG remains open.

More generally, causal abstractions that induce only a small number of maximal graphs appear favorable from the standpoint of tractable identifiability analysis.

\section{Conclusion}
\label{sec:conclusion}

In this work, we proposed a unified framework for studying causal identifiability in causal abstractions represented as collections of causal graphs. Within this setting, we formalized several notions of identifiability, namely IG, ICD, IGP, ICA, ICB, and ICF, providing a common language for comparing previously disparate definitions.

We established the hierarchy relating these notions (Figure~\ref{fig:liens_notions_id}) and characterized its strictness through explicit counterexamples. These results delineate the relative strength of the different notions and show that several natural implications fail in general.

To analyze these relationships, we used derivation graphs and step graphs as methodological tools for reasoning about causal identifiability across collections of graphs. These constructions provide a fine‑grained representation of causal derivations that hold and help explain why existing approaches, including atomically complete calculi, do not necessarily capture all identifiable queries in causal abstractions.

From a broader perspective, our results reveal a fundamental separation between graph-based and derivation-based notions of identifiability. This separation raises both methodological and computational questions. In particular, developing efficient procedures for deciding IG and ICD, and characterizing their computational complexity, remain open directions. More generally, the hierarchy provides a basis for understanding and designing causal abstractions with controlled identifiability properties.

\newpage
\bibliography{main}
\bibliographystyle{tmlr}

\appendix
\section{Proofs of Section~\ref{sec:strict_hierarchy}}
\subsection{Proofs of Section \ref{subsec:implications}}
\label{sec:proof:prop}

\theoremhierarchy*
\begin{proof}
We prove each implication separately.

\begin{enumerate}

    \item \textbf{ICB $\Rightarrow$ ICA.}
    Assume that ICB holds for $P(\mathbf{y}\mid \Do(\mathbf{x}))$ over $\mathcal{C}$. Then, there exists a set $\mathbf{Z}$ that satisfies the backdoor criterion relative to $\mathbf X$ and $\mathbf Y$ in all $\G \in \mathcal C$. Since the backdoor criterion is a special case of the adjustment criterion, $\mathbf{Z}$ also satisfies the adjustment criterion relative to $\mathbf{X}$ and $\mathbf{Y}$, in all  $\G \in \mathcal C$.
    Therefore, ICA holds.

    \item \textbf{ICB $\Rightarrow$ ICD.}
    Assume that ICB holds for
    $
    P(\mathbf{y}\mid \Do(\mathbf{x}))
    $
    over $\mathcal{C}$, and let $\mathbf{Z}$ be the corresponding common backdoor adjustment set.  For every $\mathcal{G}\in\mathcal{C}$, the backdoor criterion yields:

    \begin{align*}
        P(\mathbf y \mid \Do(\mathbf x)) &= \sum_{\mathbf z}P(\mathbf y,\mathbf z \mid \Do(\mathbf x))  &(\text{by marginalisation})\\
            &= \sum_{\mathbf z}P(\mathbf y \mid \Do(\mathbf x), \mathbf z) P(\mathbf z \mid \Do(\mathbf x)) &(\text{by chain rule})\\
            &= \sum_{\mathbf z}P(\mathbf y \mid \Do(\mathbf x), \mathbf z) P(\mathbf z) &(\text{by } \mathcal R3)\\
            &= \sum_{\mathbf z}P(\mathbf y \mid \mathbf x, \mathbf z) P(\mathbf z) &(\text{by } \mathcal R2)\\
    \end{align*}
    where:
    \begin{itemize}
        \item the equality
        $
        P(\mathbf{z}\mid \Do(\mathbf{x})) = P(\mathbf{z})
        $
        follows from Rule~3 of do-calculus since
        $
        \mathbf{Z}
        $
        contains no descendant of $\mathbf{X}$;
        \item the equality
        $
        P(\mathbf{y}\mid \Do(\mathbf{x}),\mathbf{z})
        =
        P(\mathbf{y}\mid \mathbf{x},\mathbf{z})
        $
        follows from Rule~2 of do-calculus since
        $
        \mathbf{Z}
        $
        blocks all backdoor paths from $\mathbf{X}$ to $\mathbf{Y}$.
    \end{itemize}

    Since this derivation holds for every graph in $\mathcal{C}$, it is a common derivation. Hence ICD holds.

    \item \textbf{ICF $\Rightarrow$ ICD.}
    Assume that ICF holds for
    $
    P(\mathbf{Y}\mid \Do(\mathbf{X}))
    $
    over $\mathcal{C}$, and let $\mathbf{Z}$ be the corresponding common frontdoor set. For every $\mathcal{G}\in\mathcal{C}$, the frontdoor criterion yields:
    \begin{align*}
        P(\mathbf{y}\mid \Do(x))
        &= \sum_{\mathbf z}P(\mathbf y,\mathbf z \mid \Do(x))  &(\text{by marginalisation})\\
        &= \sum_{\mathbf{z}}
        P(\mathbf{z}\mid \Do(x))
        P(\mathbf{y}\mid \mathbf{z},\Do(x)) &(\text{by chain rule})\\
        &=
        \sum_{\mathbf{z}}
        P(\mathbf{z}\mid x)
        P(\mathbf{y}\mid \mathbf{z},\Do(x)) &(\text{by }\mathcal R2)\\
        &=
        \sum_{\mathbf{z}}
        P(\mathbf{z}\mid x)
        P(\mathbf{y}\mid \Do(\mathbf{z}),\Do(x)) &(\text{by }\mathcal R2)\\
        &=
        \sum_{\mathbf{z}}
        P(\mathbf{z}\mid x)
        P(\mathbf{y}\mid \Do(\mathbf{z})) &(\text{by }\mathcal R3)\\
        &=
        \sum_{\mathbf{z}}
        P(\mathbf{z}\mid x)
        \sum_{x'}
        P(\mathbf{y}\mid x',\mathbf{z})
        P(x'), &(\text{by backdoor derivation})
    \end{align*}
    where:
    \begin{itemize}
        \item the equality $P(\mathbf{z}\mid \Do(x)) = P(\mathbf{z}\mid x)$
        follows from Rule~2 of do-calculus since there is no unblocked path from X to $\mathcal Z$, that contains an arrow into X;
        
        \item the equality $P(\mathbf{y}\mid \mathbf{z},\Do(x)) =  P(\mathbf{y}\mid \Do(\mathbf{z}),\Do(x))$ follows from $\mathcal R2$ since all backdoor paths from $\mathbf Z$ to $Y$ are blocked by $X$;
        
        \item  the equality $P(\mathbf{y}\mid \Do(\mathbf{z}),\Do(x)) =  P(\mathbf{y}\mid \Do(\mathbf{z}))$ follows from $\mathcal R3$ since $\mathbf Z$ intercepts all directed paths from $X$ to $Y$;
        
        \item the equality
        $ P(\mathbf{y}\mid \Do(\mathbf{z})) = \sum_{x'} P(\mathbf{y}\mid x',\mathbf{z}) P(x')$
        follows from the backdoor criterion since all backdoor paths from $\mathbf Z$ to $Y$ are blocked by $X$ and since there is no collider-free backdoor path from X to $\mathcal Z$, we know that X cannot be a descendant of $\mathbf Z$. The case (ICB $\Rightarrow$ ICD) gives us the corresponding derivation.
    \end{itemize}

    Since this derivation holds for every graph in $\mathcal{C}$, it is a common derivation. Hence ICD holds.

    \item \textbf{ICA $\Rightarrow$ IG.}
    Assume that ICA holds for $P(\mathbf{y}\mid \Do(\mathbf{x}))$ over $\mathcal{C}$, and let $\mathbf{Z}$ be the corresponding common adjustment set.

    By Theorem~\ref{th:adjustment_perkovic}, for every graph $ \mathcal{G}\in\mathcal{C},$ and for all probability $P$ compatible with $G$, we have:
    \[
    P(\mathbf{y} \mid \Do(\mathbf{x})) =
    \begin{cases}
        P(\mathbf{y} \mid \mathbf{x}) & \text{if } \mathbf{Z} = \emptyset, \\
        \sum_{\mathbf{z}} P(\mathbf{y} \mid \mathbf{x}, \mathbf{z}) P(\mathbf{z}) & \text{otherwise.}
    \end{cases}
    \]

    Therefore, the identification functional $\displaystyle     P_{\mathcal M} \mapsto
    \begin{cases}
    P_{\mathcal M}(\mathbf{y} \mid \mathbf{x}) & \text{if } \mathbf{Z} = \emptyset, \\
    \sum_{\mathbf{z}} P_{\mathcal M}(\mathbf{y} \mid \mathbf{x}, \mathbf{z}) P_{\mathcal M}(\mathbf{z}) & \text{otherwise.}
    \end{cases}$ is valid for all $\mathcal M$ such that $\G_{\mathcal M} \in \mathcal C$. Hence IG holds.

    \item \textbf{ICD $\Rightarrow$ IG.} This implication is shown by Proposition~\ref{prop:ICD_implique_IG}.

    \item \textbf{IG $\Rightarrow$ IGP.} This implication is shown by Proposition~\ref{prop:IG_implique_IGP}.

\end{enumerate}
\end{proof}

\subsection{Proofs of Section~\ref{subsec:strictness}}
\label{proof:strictness}

While Theorem~\ref{th:liens_notions_id} establishes a hierarchy of implications between these identifiability notions, Proposition~\ref{prop:strictness} demonstrates that these implications are strict and that some notions are incomparable:

\propstrictness*
\begin{proof}
We consider the converses of the implications presented in Theorem~\ref{th:liens_notions_id}:
\begin{itemize}
    \item \textbf{The failure of $ICA \Rightarrow ICB$} is given by the following counter example: Consider the graph $\mathcal G \coloneqq Z \leftarrow X \rightarrow Y $. $\{Z\}$ satisfies the adjustment criterion relative to $X$ and $Y$ in $G$ but does not satisfy the backdoor criterion since $Z$ is a descendant of $X$. Hence, $ICA$ holds for $P(y \mid \Do(x))$ in $\{\G\}$ but $ICB$ does not  hold for $P(y \mid \Do(x))$ in $\{\G\}$. 
    
    \item \textbf{Failures of $ICD \Rightarrow ICA$, $IG \Rightarrow ICA$, $ICD \Rightarrow ICB$ and $ICD \Rightarrow ICF$ from the single graph setting.} We consider the collection $\mathcal C$ made of a single graph, the Napkin Graph presented in Figure~\ref{fig:napkin_graph} and we consider the query $P(y\mid \Do(x))$. We have the following properties:
        \begin{itemize}
            \item As shown in \citep{pearl2018bookofwhy, guo2025causalinferencenapkingraph}, $P(y\mid \Do(x))$ is identifiable in the Napkin Graph. For completeness, identifiability is given by the following derivation:
            \begin{align*}
                P(y\mid \Do(x)) &= P(y\mid \Do(x),\Do(z))
                                &(\text{by }\mathcal R3)\\
                                &= P(y\mid x,\Do(z))
                                &(\text{by }\mathcal R2)\\
                                &= \frac{P(y,x \mid \Do(z))}{P(x \mid \Do(z))}
                                &(\text{conditional probability})\\
                                &= \frac{\sum_w P(y,x \mid z,w)P(w)}{\sum_w P(x \mid z,w)P(w)}.
                                &(\text{by backdoor derivation})
            \end{align*}
            Hence, ICD and IG hold in $\C$ for $P(y\mid \Do(x))$.
            
            \item We will show that ICB, ICA and ICF do not hold in $\C$. We consider the cases separately:
            \begin{itemize}
                \item \textbf{ICB and ICA}. All candidate adjustment sets are $\left\{ \emptyset, \{W\}, \{Z\}, \{W,Z\} \right\}$. However, none of them is an adjustment set. Indeed, conditioning on $W$ or $Z$ opens the backdoor path $X \dashleftrightarrow W \dashleftrightarrow Y$ and not conditioning on them lets the path $X \leftarrow Z \leftarrow W \dashleftrightarrow Y$ open. Hence, no common adjustment set exists and ICA does not hold. Since ICB is a special case of ICA, ICB does not hold either
                
                \item \textbf{ICF.} All candidate frontdoor sets are $\left\{ \emptyset, \{W\}, \{Z\}, \{W,Z\} \right\}$. However, none of them is an adjustment set since it is impossible to intercept $X \rightarrow Y$.
            \end{itemize}
            Hence, in this case, ICB, ICA, and ICF all fail to hold.
        \end{itemize}
    
    \item \textbf{Incomparability of ICA and ICD:} The failure of $ICD \Rightarrow ICA$ has been established in the previous item and Proposition~\ref{prop:ICA_not_ICD} establishes the failure of $ICA \Rightarrow ICD$.
    
    \item \textbf{The failure of $ICA \Rightarrow ICD$} is established by Proposition~\ref{prop:ICA_not_ICD}
    
    \item \textbf{The failure of $IG \Rightarrow ICD$} is established by Corollary~\ref{cor:IG_not_ICD}
    
    \item  \textbf{The failure of $IGP \Rightarrow IG$} is established by Proposition~\ref{ex:failureIGPfIG}.
\end{itemize}
\end{proof}

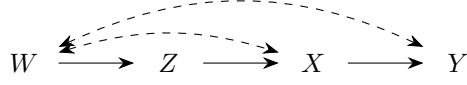
\begin{figure}[t]
    \centering
\begin{tikzpicture}[baseline=-0.5ex, node distance=10mm,
      node/.style = {inner sep=.4pt, minimum size=9mm,},
      >={Stealth[length=6pt]}, every edge/.style={->, semithick}, scale=2]
      
      \node[node] (W) {$W$};
      \node[node] (Z) [right=of W] {$Z$};
      \node[node] (X) [right=of Z] {$X$};
      \node[node] (Y) [right=of X] {$Y$};
      
      \draw[->] (W) -- (Z);
      \draw[->] (Z) -- (X);
      \draw[->] (X) -- (Y);
      
      \draw[dashed,<->,bend left=17] (W) to (X);
      \draw[dashed,<->,bend left=25] (W) to (Y);

    \end{tikzpicture}
\caption{The Napkin Graph \citep{pearl2018bookofwhy, guo2025causalinferencenapkingraph}}
\label{fig:napkin_graph}
\end{figure}

\subsubsection{Proofs of Section~\ref{subsubsec:simple_sep}}
\label{sec:proof:simple_sep}

\failureIGPfIG*

\begin{proof}[Detailed Proof]
    We prove the failure of $IGP \Rightarrow IG$, by considering the collection of causal diagrams {$\C \coloneqq \{ X \rightarrow Y, X \leftarrow Y \}$} and the causal query $P(Y = 1 \mid \Do(X = 1))$. We have the following properties:
    \begin{itemize}
        \item $P(Y = 1 \mid \Do(X = 1))$ \textbf{is not identifiable through graphs in $\C$.} Indeed, let us consider the SCM $\M_1$ defined by $X \sim \mathcal{B}\left(\frac{1}{2}\right)$ and $Y = X \oplus U$ where $U \sim \mathcal{B}\left(\frac{1}{10}\right)$ and $\oplus$ denotes the XOR operator. Moreover, let us consider the SCM $\M_2$ defined by $Y \sim \mathcal{B}\left(\frac{1}{2}\right)$ and $X = Y \oplus U$ where $U \sim \mathcal{B}\left(\frac{1}{10}\right)$. We have $\G_{\mathcal M_1} = X \rightarrow Y$ and $\G_{\mathcal M_2} = X \leftarrow Y$. Then, one can verify that:
        \[
        P_{\M_1}(x, y) = \frac{9}{20} \cdot \mathbf{1}\{y = x\} + \frac{1}{20} \cdot \mathbf{1}\{y \ne x\} = P_{\M_2}(x, y)
        \]
        However, $P_{\M_1}(Y = 1 \mid \Do(X = 1)) = \frac{9}{10}$ and  $P_{\M_2}(Y = 1 \mid \Do(X = 1)) = \frac{1}{2}$. 
        Therefore, $P(Y = 1 \mid \Do(X = 1))$ is not identifiable through graphs in $\C$.
        
        \item However, $P(Y = 1 \mid \Do(X = 1))$ \textbf{may be identifiable through graphs knowing $P^\star(X, Y)$ in $\C$.} Indeed, let us consider the case where $X$ and $Y$ are independent in the true data distribution $P^\star(X,Y)$. Let $\M_1$ be an SCM inducing $X \rightarrow Y$ such that $P_{\M_1}(X,Y) = P^\star(X,Y)$. We have $P_{\M_1}(Y = 1 \mid \Do(X = 1)) = P_{\M_1}(Y = 1 \mid X = 1) = P^\star(Y = 1 \mid X = 1) = P^\star(Y=1)$. Similarly, let $\M_2$ be an SCM inducing $X \leftarrow Y$ such that $P_{\M_2}(X,Y) = P^\star(X,Y)$. We have $P_{\M_2}(Y = 1 \mid do(X=1)) = P_{\M_2}(Y = 1)= P^\star(Y=1)$. Therefore the identification functional $\F \coloneqq P(X,Y) \mapsto P(Y = 1)$ is valid  in all SCM $\M$ satisfying \(\G_{\M} \in \C\), $P_M(X,Y) = P^\star(X,Y)$, and  $P(Y = 1 \mid \Do(X = 1))$ is identifiable  through graphs knowing $P^\star(X, Y)$ in $\C$.
        \end{itemize}
        Hence, in this case, IGP holds but IG does not hold.
\end{proof}

\subsection{Proofs of Section~\ref{subsubsec:ICDvsIG}}
\label{sec:proof:ICDvsIG}

\propICAnotICD*

\begin{proof}
    Consider the graph $\G_1 \coloneqq Z \leftarrow X \rightarrow Y$ and $\G_2 \coloneqq X \leftarrow Z \rightarrow Y$, and the class $\C \coloneqq \{ \G_1, \G_2 \}$. We have the following properties:
    \begin{itemize}
        \item $\{Z\}$ satisfies the adjustment criterion in both $\G_1$ and $\G_2$, hence ICA holds in $\C$ for $P(y \mid \Do(x))$.

        \item Consider the graph $\G_3 \coloneqq X \rightarrow Z \rightarrow Y$. One can verify that:
        \[
        \mathcal{D}[\G_1] \cap \mathcal{D}[\G_2] \subseteq \mathcal{D}[\G_3].
        \]

        This inclusion was established by exhaustively enumerating all edges in \(\mathcal D[\G_1]\) and \(\mathcal D[\G_2]\), using the code of \cite{yvernes2026unveiling}. This enumeration produces the complete edge sets shown in Figures~\ref{fig:dg-int} and~\ref{fig:dg-non-ia}, representing \(\mathcal D[\G_1]\cap\mathcal D[\G_2]\) and \(\mathcal D[\G_3]\), respectively. Thus, the claimed inclusion is both computationally verified and directly checkable from the complete graphs displayed in the figures.

        Assume, by contradiction, that $P(y \mid \Do(x))$ is ICD in $\C$. Then, there exists a derivation $D$ that identifies $P(y \mid \Do(x))$ in all $\G \in \C$. Hence:
        \[
        \mathcal{S}[D] \subseteq \mathcal{D}[\G_1] \cap \mathcal{D}[\G_2] \subseteq \mathcal{D}[\G_3].
        \]
        Therefore, by Theorem~\ref{thm:derivation_validity}, $D$ holds in $\G_3$, hence $P(y \mid \Do(x))$ is ICD, thus IG (by Proposition~\ref{prop:ICD_implique_IG}), in $\{\G_1, \G_2, \G_3\}$.

        However, since $\{Z\}$ does not satisfy the adjustment criterion in $\G_3$, by Theorem~\ref{th:adjustment_perkovic}, there exists a density $P^\star$ Markovian with respect to $\G_3$ such that:
        \[
        P^\star(y \mid \Do(x)) \neq \sum_{z} P^\star(y\mid z,x)P^\star(z).
        \]
        The construction of $P^\star$, given in the proof of Theorem~57 in \cite{Perkovic_2018}, yields an SCM $\mathcal{M}_3$ with $P_{\mathcal{M}_3} = P^\star$ and $\G_{\mathcal{M}_3} = \G_3$. More precisely, $P_{\mathcal{M}_3}$ is a linear gaussian SCM, hence $P^\star$ as a density with respect to Lebesgue measure. Although $P^\star$ corresponds to a continuous distribution, we use the discrete notation adopted throughout the paper for simplicity. The arguments remain unchanged when expressed in terms of densities and integrals instead of probability masses and sums.
        
        Moreover, since $P^\star$ is Markovian with respect to $\G_3$, thus by Bayes' rule we have:
        \[
            P^\star(x,y,z) = P^\star(y \mid z) P^\star(z \mid x) P^\star(x) = P^\star(y \mid z) P^\star(x \mid z) P^\star(z)
        \]
        Hence $P^\star$ is also Markovian with respect to $\G_2$, hence by Proposition~7.1 from \cite{PetJanSch17}, there exists an SCM $\mathcal{M}_2$ with $P_{\mathcal{M}_2} = P^\star$ and $\G_{\mathcal{M}_2} = \G_2$.\footnote{The fact that $P^\star$ as a density with respect to Lebesgue measure is needed here.} In $\G_2$, since $\{Z\}$ satisfies the adjustment criterion, we have:
        \[
        P^\star(y \mid \Do(x)) = \sum_{z} P^\star(y\mid z,x)P^\star(z).
        \]
        \textbf{This contradicts the previous inequality}.
        Therefore, $P(y \mid \Do(x))$ cannot be ICD in $\C$.
    \end{itemize}
    Hence, $ICA \Rightarrow ICD$ fails.
\end{proof}

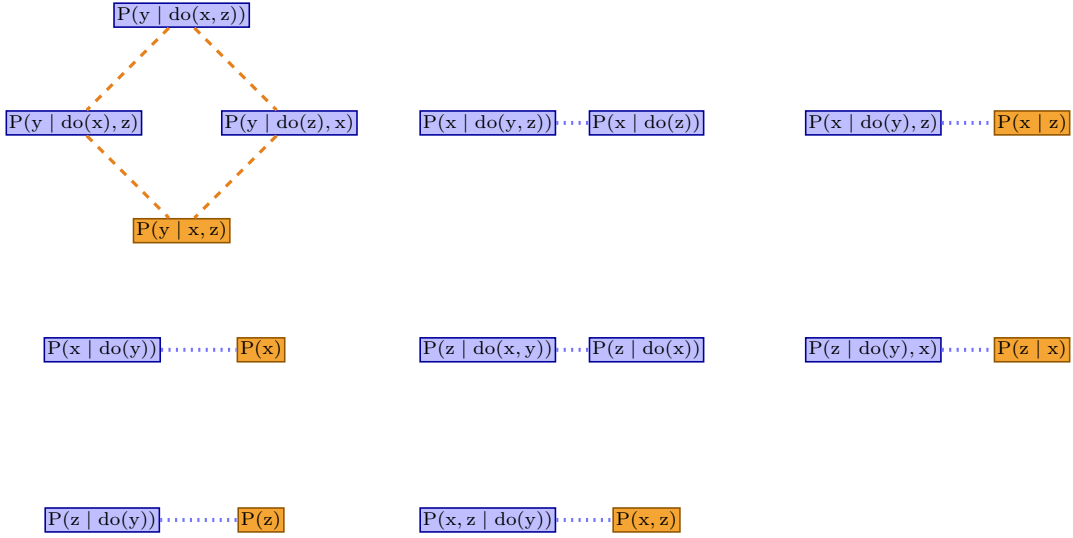
\begin{figure}
\centering
\begin{tikzpicture}[
    dofree/.style = {
      draw=coolorange!60!black,
      fill=coolorange!80,
      line width=0.6pt,
      inner sep=1pt,
      font=\scriptsize
    },
    withdo/.style = {
      draw=blue!60!black,
      fill=blue!25,
      line width=0.6pt,
      inner sep=1pt,
      font=\scriptsize
    },
    styleR1/.style = {thick, -, gray, line width=1.2pt},
    styleR2/.style = {thick, -, orange!60!brown, dashed, line width=1.2pt},
    styleR3/.style = {thick, -, blue!55, dotted, line width=1.5pt},
    scale=3
 ]
\node[withdo] (n0) at (0.000, 0.475) {$\mathrm{P(y \mid \Do(x,z))}$};
\node[withdo] (n1) at (-0.475, 0.000) {$\mathrm{P(y \mid \Do(x), z)}$};
\node[withdo] (n2) at (0.475, -0.000) {$\mathrm{P(y \mid \Do(z), x)}$};
\node[dofree] (n3) at (-0.000, -0.475) {$\mathrm{P(y \mid x,z)}$};

\node[withdo] (n4) at (1.350, 0.000) {$\mathrm{P(x \mid \Do(y,z))}$};
\node[withdo] (n5) at (2.050, 0.000) {$\mathrm{P(x \mid \Do(z))}$};

\node[withdo] (n6) at (3.050, 0.000) {$\mathrm{P(x \mid \Do(y), z)}$};
\node[dofree] (n7) at (3.750, 0.000) {$\mathrm{P(x \mid z)}$};

\node[withdo] (n8) at (-0.350, -1.00) {$\mathrm{P(x \mid \Do(y))}$};
\node[dofree] (n9) at (0.350, -1.00) {$\mathrm{P(x)}$};

\node[withdo] (n10) at (1.350, -1.00) {$\mathrm{P(z \mid \Do(x,y))}$};
\node[withdo] (n11) at (2.050, -1.00) {$\mathrm{P(z \mid \Do(x))}$};

\node[withdo] (n12) at (3.050, -1.00) {$\mathrm{P(z \mid \Do(y), x)}$};
\node[dofree] (n13) at (3.750, -1.00) {$\mathrm{P(z \mid x)}$};

\node[withdo] (n14) at (-0.350, -1.75) {$\mathrm{P(z \mid \Do(y))}$};
\node[dofree] (n15) at (0.350, -1.75) {$\mathrm{P(z)}$};

\node[withdo] (n16) at (1.350, -1.75) {$\mathrm{P(x,z \mid \Do(y))}$};
\node[dofree] (n17) at (2.050, -1.75) {$\mathrm{P(x,z)}$};

\draw[styleR3] (n4) -- (n5);
\draw[styleR3] (n6) -- (n7);
\draw[styleR3] (n8) -- (n9);
\draw[styleR2] (n0) -- (n2);
\draw[styleR2] (n0) -- (n1);
\draw[styleR2] (n1) -- (n3);
\draw[styleR2] (n2) -- (n3);
\draw[styleR3] (n10) -- (n11);
\draw[styleR3] (n12) -- (n13);
\draw[styleR3] (n14) -- (n15);
\draw[styleR3] (n16) -- (n17);
\end{tikzpicture}
\caption{Intersection of the derivations graphs of $\G_1$ and $\G_2$. Blue nodes denote interventional queries, while orange nodes denote do-free queries. Edges represent applications of the do-calculus rules: gray solid lines for Rule $\mathcal{R}1$, orange dashed lines for Rule $\mathcal{R}2$, and blue dotted lines for Rule $\mathcal{R}3$. Only non-isolated queries are displayed.}
\label{fig:dg-int}
\end{figure}

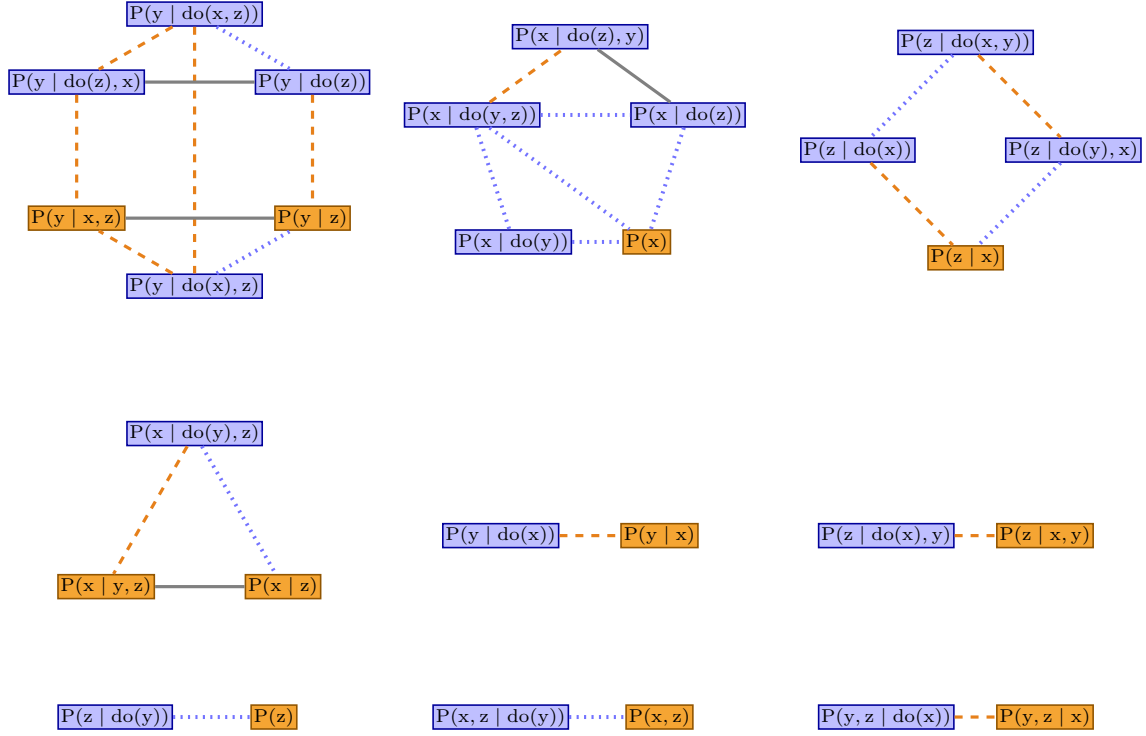
\begin{figure}
\centering
\begin{tikzpicture}[
    dofree/.style = {
      draw=coolorange!60!black,
      fill=coolorange!80,
      line width=0.6pt,
      inner sep=1pt,
      font=\scriptsize
    },
    withdo/.style = {
      draw=blue!60!black,
      fill=blue!25,
      line width=0.6pt,
      inner sep=1pt,
      font=\scriptsize
    },
    styleR1/.style = {thick, -, gray, line width=1.2pt},
    styleR2/.style = {thick, -, orange!60!brown, dashed, line width=1.2pt},
    styleR3/.style = {thick, -, blue!55, dotted, line width=1.5pt},
    scale=3
 ]
\node[withdo] (n0) at (0.000, 0.600) {$\mathrm{P(y \mid \Do(x,z))}$};
\node[withdo] (n1) at (-0.000, -0.600) {$\mathrm{P(y \mid \Do(x), z)}$};
\node[withdo] (n2) at (-0.520, 0.300) {$\mathrm{P(y \mid \Do(z), x)}$};
\node[withdo] (n3) at (0.520, 0.300) {$\mathrm{P(y \mid \Do(z))}$};
\node[dofree] (n4) at (-0.520, -0.300) {$\mathrm{P(y \mid x,z)}$};
\node[dofree] (n5) at (0.520, -0.300) {$\mathrm{P(y \mid z)}$};
\node[withdo] (n6) at (1.224, 0.155) {$\mathrm{P(x \mid \Do(y,z))}$};
\node[withdo] (n7) at (1.406, -0.405) {$\mathrm{P(x \mid \Do(y))}$};
\node[withdo] (n8) at (1.700, 0.500) {$\mathrm{P(x \mid \Do(z), y)}$};
\node[withdo] (n9) at (2.176, 0.155) {$\mathrm{P(x \mid \Do(z))}$};
\node[dofree] (n10) at (1.994, -0.405) {$\mathrm{P(x)}$};
\node[withdo] (n11) at (3.400, 0.475) {$\mathrm{P(z \mid \Do(x,y))}$};
\node[withdo] (n12) at (2.925, 0.000) {$\mathrm{P(z \mid \Do(x))}$};
\node[withdo] (n13) at (3.875, -0.000) {$\mathrm{P(z \mid \Do(y), x)}$};
\node[dofree] (n14) at (3.400, -0.475) {$\mathrm{P(z \mid x)}$};
\node[withdo] (n15) at (0.000, -1.250) {$\mathrm{P(x \mid \Do(y), z)}$};
\node[dofree] (n16) at (-0.390, -1.925) {$\mathrm{P(x \mid y,z)}$};
\node[dofree] (n17) at (0.390, -1.925) {$\mathrm{P(x \mid z)}$};
\node[withdo] (n18) at (1.350, -1.700) {$\mathrm{P(y \mid \Do(x))}$};
\node[dofree] (n19) at (2.050, -1.700) {$\mathrm{P(y \mid x)}$};
\node[withdo] (n20) at (3.050, -1.700) {$\mathrm{P(z \mid \Do(x), y)}$};
\node[dofree] (n21) at (3.750, -1.700) {$\mathrm{P(z \mid x,y)}$};
\node[withdo] (n22) at (-0.350, -2.500) {$\mathrm{P(z \mid \Do(y))}$};
\node[dofree] (n23) at (0.350, -2.500) {$\mathrm{P(z)}$};
\node[withdo] (n24) at (1.350, -2.500) {$\mathrm{P(x,z \mid \Do(y))}$};
\node[dofree] (n25) at (2.050, -2.500) {$\mathrm{P(x,z)}$};
\node[withdo] (n26) at (3.050, -2.500) {$\mathrm{P(y,z \mid \Do(x))}$};
\node[dofree] (n27) at (3.750, -2.500) {$\mathrm{P(y,z \mid x)}$};
\draw[styleR2] (n6) -- (n8);
\draw[styleR3] (n6) -- (n7);
\draw[styleR3] (n6) -- (n9);
\draw[styleR2] (n15) -- (n16);
\draw[styleR3] (n15) -- (n17);
\draw[styleR3] (n7) -- (n10);
\draw[styleR1] (n8) -- (n9);
\draw[styleR1] (n16) -- (n17);
\draw[styleR3] (n9) -- (n10);
\draw[styleR2] (n0) -- (n2);
\draw[styleR2] (n0) -- (n1);
\draw[styleR3] (n0) -- (n3);
\draw[styleR2] (n1) -- (n4);
\draw[styleR3] (n1) -- (n5);
\draw[styleR2] (n18) -- (n19);
\draw[styleR1] (n2) -- (n3);
\draw[styleR2] (n2) -- (n4);
\draw[styleR1] (n4) -- (n5);
\draw[styleR2] (n3) -- (n5);
\draw[styleR2] (n11) -- (n13);
\draw[styleR3] (n11) -- (n12);
\draw[styleR2] (n20) -- (n21);
\draw[styleR2] (n12) -- (n14);
\draw[styleR3] (n13) -- (n14);
\draw[styleR3] (n22) -- (n23);
\draw[styleR3] (n24) -- (n25);
\draw[styleR2] (n26) -- (n27);
\draw[styleR3] (n6) -- (n10);
\end{tikzpicture}
\caption{Derivation graph of $\G_3$. Blue nodes denote interventional queries, while orange nodes denote do-free queries. Edges represent applications of the do-calculus rules: gray solid lines for Rule $\mathcal{R}1$, orange dashed lines for Rule $\mathcal{R}2$, and blue dotted lines for Rule $\mathcal{R}3$. Only non isolated queries are displayed.}
\label{fig:dg-non-ia}
\end{figure}

\newpage
\section{Proofs of Section~\ref{sec:discussions}}
\label{proof:further}

\lemmamonotonicity*

\begin{proof}
    Let us prove the different points:
    \begin{itemize}
        
        \item[1.]  Let $\F_1$ be an identification functional valid in $\G_1$ for \(P\bigl(\mathbf{y}\mid \Do(\mathbf{x}),\,\mathbf{z}\bigr)\). Let $\M_2$ be an SCM inducing $\G_2$. By adding missing arguments in the structural functions of $\M_2$, we can construct an SCM $\M_1$ inducing $\G_1$ such that $P_{\M_1}(\V) = P_{\M_2}(\V)$ and $P_{\M_1} \bigl(\mathbf{y} \mid \Do(\mathbf{x}), \mathbf{z}\bigr) = P_{\M_2}\bigl(\mathbf{y} \mid \Do(\mathbf{x}), \mathbf{z}\bigr)$:
        
        Let $\E_B^+ \coloneqq \E_B(\G_1) \setminus \E_B(\G_2)$ be the set of bidirected edges in $\G_1$ that are absent from $\G_2$. For each $V_i \dashleftrightarrow V_j \in \E_B^+$, introduce a fresh exogenous variable $U_{ij}$. Let $\U^+$ denote the set of these new variables and define the exogenous variables of $\M_1$ by $\U^{(1)} \coloneqq \U^{(2)} \cup \U^+$. We retain the exogenous variables and their joint distribution from $\M_2$, and take each variable in $\U^+$ independently of them and of all other added variables. Thus, $\M_1$ has the same endogenous variables $\V$ and the same original exogenous variables as $\M_2$, augmented by the new variables in $\U^+$. For each $V_i \in \V$, let $\U_i^+$ be the set of new exogenous variables corresponding to bidirected edges incident to $V_i$, and set $\U_i^{(1)} \coloneqq \U_i^{(2)} \cup \U_i^+$. 

        Since $\G_2 \subseteq \G_1$, for all $V_i \in \V$ we have $\Pa(V_i,\G_2) \subseteq \Pa(V_i,\G_1)$. For any variable $A$, let $\mathcal{X}_A$ denote its valuation space. We define the structural functions $f_i^{(1)} \in \F_{\M_1}$ for each $V_i \in \V$ as:
        \[
        \begin{aligned}
        f_i^{(1)} &: \left\{
        \begin{aligned}
        \mathcal{X}_{\Pa(V_i,\G_1)} \times \mathcal{X}_{\U_i^{(2)}} &\to \mathcal{X}_{V_i},\\
        (x_{\Pa(V_i,\G_1)},u_i^{(2)}) &\mapsto f_i^{(2)}(x_{\Pa(V_i,\G_2)},u_i^{(2)})
        \end{aligned}
        \right.
        && \text{if } \U_i^+ = \emptyset, \\
        f_i^{(1)} &: \left\{
        \begin{aligned}
        \mathcal{X}_{\Pa(V_i,\G_1)} \times \mathcal{X}_{\U_i^{(2)}} \times \mathcal{X}_{\U_i^+} &\to \mathcal{X}_{V_i},\\
        (x_{\Pa(V_i,\G_1)},u_i^{(2)},u_i^+) &\mapsto f_i^{(2)}(x_{\Pa(V_i,\G_2)},u_i^{(2)})
        \end{aligned}
        \right.
        && \text{if } \U_i^+ \neq \emptyset.
        \end{aligned}
        \]
        where $f_i^{(2)} \in \F_{\M_2}$ is the corresponding structural function in $\M_2$, and in the second case $u_i^+$ assigns values to the newly added exogenous variables in $\U_i^+$. In both cases, $f_i^{(1)}$ ignores the added directed-parent arguments in $\Pa(V_i,\G_1) \setminus \Pa(V_i,\G_2)$; in the second case, it also ignores $u_i^+$. Each new $U_{ij}$ is a shared exogenous parent of $V_i$ and $V_j$, so it induces the required dashed bidirected edge. Since all added arguments are ignored, $\M_1$ and $\M_2$ induce the same observational and interventional distributions, while $\M_1$ induces $\G_1$.

        Hence, $\F_1(P_{\M_2}(\V)) = \F_1(P_{\M_1}(\V)) =  P_{\M_1}\bigl(\mathbf{y} \mid \Do(\mathbf{x}), \mathbf{z}\bigr) = P_{\M_2} \bigl(\mathbf{y} \mid \Do(\mathbf{x}), \mathbf{z}\bigr)$. Therefore, $\F_1$ is valid in $\G_2$.
        
        \item[2.] The construction above preserves the observational distribution and the interventional distribution. Hence, if $P_{\M_2}(\V)=P^\star(\V)$, then $P_{\M_1}(\V)=P^\star(\V)$ as well. identification functional valid for \(P\bigl(\mathbf{y}\mid \Do(\mathbf{x}), \mathbf{z}\bigr)\) in all SCMs $\mathcal M_1$ inducing \(\G_1\) and satisfying \(P_{\mathcal M_1}(\mathcal{V})=P^\star(\mathcal{V})\) is also valid in all SCMs $\mathcal M_2$ inducing \(\G_2\) and satisfying \(P_{\mathcal M_2}(\mathcal{V})=P^\star(\mathcal{V})\).
    
        \item[3 \& 4.] Since $\G_2$ is a subgraph of $\G_1$, every graphical independence that holds in $\G_1$ also holds in $\G_2$. Therefore, any causal derivation holding in $\G_1$, holds in  $\G_2$ and any graphical criterion satisfied in $\G_1$ is satisfied in $\G_2$.
    \end{itemize}
\end{proof}

\end{document}